\pdfoutput=1
\documentclass{article}
\usepackage{ijcai22}

\usepackage{times}
\usepackage{soul}
\usepackage{url}
\usepackage[hidelinks]{hyperref}
\usepackage[utf8]{inputenc}
\usepackage[small]{caption}
\usepackage{graphicx}
\usepackage{amsmath}
\usepackage{amsthm}
\usepackage{amssymb}
\usepackage{latexsym}
\usepackage{booktabs}
\usepackage{algorithm}
\usepackage{algorithmic}
\usepackage{color}
\usepackage{relsize}
\usepackage{enumitem}
\newtheorem{example}{Example}
\newtheorem{theorem}{Theorem}
\newtheorem{definition}{Definition}
\newtheorem{prop}{Proposition}

\newtheorem{corollary}{Corollary}

\newcommand{\amendb}[1]{{\color{black}{#1}}}
\newcommand{\amendc}[1]{{\color{black}{#1}}}

\newcommand{\Reg}{\mathcal{R}}
\newcommand{\at}{\mathcal{D}}

\newcommand{\ie}{{\it i.e.}}

\title{Automatic Verification of Sound Abstractions for Generalized Planning}

\author{
Zhenhe Cui
\and
Weidu Kuang\and
Yongmei Liu
\affiliations
Dept. of Computer Science, Sun Yat-sen University, Guangzhou 510006, China
\emails
\{cuizhh3, kuangwd\}@mail2.sysu.edu.cn,
ymliu@mail.sysu.edu.cn
}

\begin{document}

\maketitle

\begin{abstract}

  Generalized planning studies the computation of general solutions  for a set of planning problems. Computing general solutions with correctness guarantee has long been a key issue in generalized planning. Abstractions are widely used to solve generalized planning problems. Solutions of sound abstractions are those with correctness guarantees for generalized planning problems. Recently, Cui et al. proposed a uniform abstraction framework for generalized planning. They gave the model-theoretic definitions of sound and complete abstractions for generalized planning problems. In this paper, based on Cui et al.'s work, we explore automatic verification of sound abstractions for generalized planning. We firstly present the proof-theoretic characterization for sound abstraction. Then, based on the characterization, we give a sufficient condition for sound abstractions which is first-order verifiable.  To implement it, we exploit regression extensions, and develop methods to handle counting and transitive closure. Finally, we implement a sound abstraction verification system and report experimental results on several domains.
\end{abstract}

\section{Introduction}

Generalized planning, where a single plan works for a set of similar planning problems, has received continual attention in the AI community \cite{2008Learning,2011Generalized,2018Features,AguasCJ19,2019Illanes}. For example, the general solution ``while the block $A$ is not clear, pick the top block above  $A$ and place it on the table" meets the goal $clear(A)$ no matter how many blocks the tower has. Computing general solutions with correctness guarantee has long been a key issue in generalized planning. However, so far only limited correctness results have been obtained, mainly for
 the so-called  one-dimensional problems \cite{HuL10} and extended LL-domains \cite{SrivastavaIZ11}.

Abstractions are widely used to solve generalized planning problems.
The idea is to develop an abstract model of the problem, which is easier to solve, and exploit a solution in the abstract model to find a solution in the concrete model. A  popular kind of abstract models for generalized planning are qualitative numerical planning (QNP) problems, introduced by \cite{2011Qualitative}: they showed that QNPs are decidable, and proposed a generate-and-test method to solve QNPs.
\citeauthor{2018Features} \shortcite{2018Features} proposed solving a generalized classical planning problem by abstracting it into a QNP, they showed that if the abstraction is sound, then a solution to the QNP is also a solution to the original problem.  Here, soundness of abstractions is the key to guarantee correctness of solutions to the original problem.

The automatic generation of sound abstractions for generalized planning problems has attracted the attention of researchers in recent years. \citeauthor{BonetFG19} \shortcite{BonetFG19} proposed  learning a QNP abstraction for a generalized planning problem from a sample set of instances, however, the abstraction is guarateed to be sound only for the sample instances.
 Further, \citeauthor{Bonet19} \shortcite{Bonet19} showed how to obtain  first-order formulas that define a set of instances on which the abstraction is guaranteed to be sound.
\citeauthor{2019Illanes} \shortcite{2019Illanes} considered solving a class of generalized planning problems by automatically deriving a sound QNP abstraction from an instance of the problem, by introducing a counter for each set of indistinguishable objects, however, this class of problems is too restricted.
%As far as we know, there have been no attempts to automatically check if an abstraction is sound for the original problem.

Recently, \citeauthor{2017Abstraction} \shortcite{2017Abstraction} proposed an agent abstraction framework
based on the situation calculus \cite{Reiter01} and Golog \cite{1997GOLOG}. They related a high-level action theory to a low-level action theory by the notion of a refinement mapping, which specifies
how each high-level action is implemented by a low-level
Golog program and how each high-level fluent can
be translated into a low-level formula.
Based on their work,
\citeauthor{Cui21} \shortcite{Cui21} proposed a uniform abstraction framework for generalized planning.
They formalized a generalized planning problem as a triple of a basic action theory, a trajectory constraint, and a goal.
They gave model-theoretic definitions of sound/complete
 abstractions of a generalized planning problem,
 and showed that solutions to a
 generalized planning problem are nicely related to those of its
 sound/complete abstractions. %Intuitively, a high-level generalized planning problem is a sound abstraction of a low-level generalized planning problem, if for any low-level planning instance $M_l$, there exist two high-level planning instance $M_h^1$ and $M_h^2$ such that any trajectory of $M_h^1$ is simulated by some refinement trajectory of $M_l$, and any refinement trajectory of $M_l$ is simulated by  some trajectory of $M_h^2$.
%Thus the refinement of any solution to the high-level generalized planning problem is a solution to the low-level generalized planning problem.
In particular, the refinement of any solution to a sound abstraction is a solution to the original problem.

The significance of the research line initiated by \citeauthor{2018Features} \shortcite{2018Features}
is that abstract problems are easier to solve, soundness of abstractions together with correctness of high-level solutions guarantee correctness of low-level solutions.
In this paper, based on Cui et al.'s work, we explore automatic verification of sound abstractions for generalized planning.
First of all, we give a proof-theoretic characterization of sound abstractions for generalized planning in the situation calculus. Based on the characterization, we give a sufficient condition for sound abstractions which is first-order verifiable.
To verify the condition with first-order theorem provers, we exploit universal and existential extensions of regression, and develop methods to handle counting and transitive closure.
Using the SMT solver Z3, we implemented a verification system and experimented with 7 domains: 5 from the literature and 2 made by us. Experimental results showed that our system was able to successfully verify soundness of abstractions for all domains in seconds.

\section{Preliminaries}

\subsection{Situation Calculus and Golog}

The situation calculus  \cite{Reiter01} is a many-sorted first-order language with some second-order ingredients suitable for describing dynamic worlds. There are three disjoint sorts: $action$ for actions, $situation$ for situations, and $object$ for everything else. The language also has the following components: a situation constant $S_{0}$ denoting the initial situation; a binary function $do(a,s)$ denoting the successor situation to $s$ resulting from performing action $a$; a binary relation $Poss(a,s)$ indicating that action $a$ is possible in situation $s$; a binary relation $s\sqsubseteq s'$, meaning situation $s$ is a sub-history of $s'$; a set of relational (functional) fluents,  \ie, predicates (functions) taking a situation term as their last argument. A formula is uniform in $s$ if it does not mention any situation term other than $s$. We call a formula with all situation arguments eliminated a situation-suppressed formula. For a situation-suppressed formula $\phi$, we use $\phi[s]$ to denote the formula obtained from $\phi$ by restoring $s$ as the situation arguments to all fluents. A situation $s$ is executable if it is possible to perform the actions in $s$ one by one: $Exec(s)\doteq \forall a,s'.do(a,s')\sqsubseteq s\supset Poss(a,s')$.

In the situation calculus, a particular domain of application can be specified by a basic action theory (BAT) of the form
\begin{center}
  $\mathcal{D}=\Sigma \cup \mathcal{D}_{ap} \cup \mathcal{D}_{ss}\cup \mathcal{D}_{una}\cup \mathcal{D}_{S_{0}}$,
\end{center}
where $\Sigma$ is the set of the foundational axioms for situations; $\mathcal{D}_{ap}$ is a set of action precondition axioms, one for each action function $A$ of the form $Poss(A(\vec{x}),s)\equiv \Pi_A(\vec{x},s)$, where $\Pi_A(\vec{x},s)$ is uniform in $s$; $\mathcal{D}_{ss}$ is a set of successor state axioms (SSAs), one for each relation fluent symbol $P$ of the form $P(\vec{x},do(a,s))\equiv \phi_{P}(\vec{x},a,s)$, and one for each functional fluent symbol $f$ of the form $f(\vec{x},do(a,s))=y\equiv \psi_{f}(\vec{x},y,a,s)$, where $\phi_{P}(\vec{x},a,s)$ and $\psi_{f}(\vec{x},y,a,s)$ are uniform in $s$; $\mathcal{D}_{una}$ is the set of unique name axioms for actions; $\mathcal{D}_{S_{0}}$ is the initial knowledge base stating facts about $S_{0}$.

%\begin{align*}
%    &Poss(mt(x),s) \equiv clear(x,s) \wedge \neg ontable(x,s); \\
%    &Poss(move(x,y),s) \equiv clear(x,s) \wedge clear(y,s) \wedge x \neq y;\\
%    %\vspace{5cm}
%    & clear(x,do(a,s)) \equiv clear(x,s) \wedge \forall y.a \neq move(y,x)
%  \\
%    & \text{\hspace*{0.5cm}} \vee \exists y.on(y,x,s) \wedge(a=mt(y) \vee \exists z.a=move(y, z));
%    \\
%    &ontable(x,do(a,s)) \equiv a=mt(x)
%    \\
%  & \text{\hspace*{0.5cm}} \vee ontable(x,s) \wedge \neg(\exists y) a=move(x, y);
%    \\
%    &on(x,y,do(a,s)) \equiv a=move(x,y) \vee on(x,y,s) \wedge\\
%    & \text{\hspace*{0.5cm}}  \neg(\exists z) a=move(x, z) \wedge a \neq mt(x).
%\end{align*}

\citeauthor{1997GOLOG} \shortcite{1997GOLOG} introduced a high-level programming language Golog with the following syntax:
\begin{center}
  $\delta ::= \alpha |\ \varphi ?\ |\ \delta_{1};\delta_{2}\ |\ \delta_{1}|\delta_{2}\ |\ \pi x.\delta\ |\ \delta^{*}$,
\end{center}
where $\alpha$ is an action term; $\varphi$ is a situation-suppressed formula and  $\varphi ?$ tests whether $\varphi$ holds; program $\delta_{1};\delta_{2}$ represents the sequential execution of $\delta_{1}$ and $\delta_{2}$; program $\delta_{1}|\delta_{2}$ denotes the non-deterministic choice between $\delta_{1}$ and $\delta_{2}$; program $\pi x.\delta$ denotes the non-deterministic choice of a value for parameter $x$ in $\delta$; program $\delta^{*}$ means executing program $\delta$ for a non-deterministic number of times.

Golog has two kinds of semantics \cite{ConGOLOG}: transition semantics and evaluation semantics. In transition semantics, a configuration of a Golog program is a pair $(\delta, s)$ of a situation $s$ and a program $\delta$ remaining to be executed. The predicate $Trans(\delta,s,\delta', s')$ means that there is a transition from configuration $(\delta,s)$ to $(\delta', s')$ in one elementary step. The predicate $Final(\delta,s)$ means that the configuration $(\delta,s)$ is a final one, which holds if program $\delta$ may legally terminate in situation $s$. In evaluation semantics, the predicate $Do(\delta, s, s')$ means that executing the program $\delta$ in situation $s$ will terminate in a situation $s'$. Do can be defined with $Trans$ and $Final$ as follows:
 $ Do(\delta,s,s')\doteq \exists \delta'. Trans^{*}(\delta,s,\delta',s')\land Final(\delta',s'),$
 where, $Trans^{*}$ denotes the transitive closure of $Trans$.

\subsection{Regression and Its Extensions}

%Regression was proposed by \citeauthor{Reiter01} \shortcite{Reiter01} to reduce the evaluation of a sentence to a first-order theorem proving task in the initial database.
Regression is an important computational mechanism for reasoning about deterministic actions and their effects in the situation calculus \cite{Reiter01}. The following is the definition of a one-step regression operator:
%then we introduce two notions of extended regression: existentially extended regression, and universally extended regression.

\begin{definition}\label{def_reg}\rm
Given a BAT $\mathcal{D}$ and a formula $\phi$. We use $\Reg_{\mathcal{D}}[\phi]$ to denote the formula obtained from $\phi$ by the following steps: for each term $f(\vec{t}, do(\alpha, \sigma))$, replace $\phi$ with $\exists y. \psi_{f}(\vec{t},y,\alpha,\sigma)\land \phi[f(\vec{t}, do(\alpha,\sigma))/y]$\footnote{$\phi[t'/t]$ denotes that formula obtained from $\phi$ by replacing all occurrences of $t'$ in $\phi$ by $t$.}; replace each atom $P(\vec{t}, do(\alpha, \sigma))$ with $\phi_P(\vec{t}, \alpha, \sigma)$;  replace each precondition atom $Poss(A(\vec{t}), \sigma)$ with $\Pi_{A}(\vec{t}, \sigma)$; and further simplify the result with $\at_{una}$.

\end{definition}

\begin{prop}
\label{prop2} $\at \models \phi \equiv \Reg_{\mathcal{D}}[\phi]$.
\end{prop}

Luo et al. \shortcite{2020Forgetting} presented the existentially extended regression,  notation $\mathcal{R}^{E}[\phi(s),\delta]$ denotes a state formula expressing that there exists an execution of program $\delta$ starting from $s$  making $\phi$ hold.

\begin{definition}\rm\label{existentially-extended-regression}
  Given a situation-suppressed formula $\phi$ and a Golog program $\delta$, the extended regression $\mathcal{R}^{E}[\phi(s),\delta]$ can be inductively defined as follows:

\begin{itemize}
  \item $\mathcal{R}^{E}[\phi(s),\alpha]= \mathcal{R}_{\mathcal{D}}[Poss(\alpha,s)\land \phi(do(\alpha,s))]$,
  \item $\mathcal{R}^{E}[\phi(s),\psi?] = \psi[s]\land \phi(s)$,
  \item $\mathcal{R}^{E}[\phi(s),\delta_{1};\delta_{2}] = \mathcal{R}^{E}[\mathcal{R}^{E}[\phi(s),\delta_{2}],\delta_{1}]$,
  \item $\mathcal{R}^{E}[\phi(s),\delta_{1}|\delta_{2}] = \mathcal{R}^{E}[\phi(s),\delta_{1}]\lor \mathcal{R}^{E}[\phi(s),\delta_{2}]$,
  \item  $\mathcal{R}^{E}[\phi(s),(\pi x)\delta(x)] = (\exists x)\mathcal{R}^{E}[\phi(s),\delta(x)]$.
\end{itemize}
\end{definition}

\begin{prop} \label{prop_Eregression}
Given a basic action theory $\mathcal{D}$, a Golog program $\delta$ and a situation-suppressed formula $\phi$, we have:
\begin{center}
  $\mathcal{D} \models \mathcal{R}^{E}[\phi(s), \delta] \equiv \exists s'.Do(\delta, s, s') \wedge \phi[s'].$
\end{center}
\end{prop}

Li and Liu \shortcite{LiL15a} presented the universally extended regression, notation $\mathcal{R}^{U}[\phi(s),\delta]$ denotes a state formula expressing that all executions of program $\delta$ starting from $s$ making $\phi$ hold. To get the definition of universally extended regression, one can replace the symbols $\mathcal{R}^{E}$, $\land$, $\lor$, and $\exists$ in Definition \ref{existentially-extended-regression} with $\mathcal{R}^{U}$,  $\supset$, $\land$, and $\forall$, respectively.

%\begin{definition}\rm
%  Given a situation-suppressed formula $\phi$ and a Golog program $\delta$, the  universally extended regression $\mathcal{R}^{U}[\phi(s),\delta]$ is defined as follows:
%
%\begin{itemize}
%  \item $\mathcal{R}^{U}[\phi(s),\alpha] = \mathcal{R}_{\mathcal{D}}[Poss(\alpha,s)\supset \phi(do(\alpha,s))];$
%  \item $\mathcal{R}^{U}[\phi(s),\psi?] = \psi[s]\supset \phi(s);$
%  \item  $\mathcal{R}^{U}[\phi(s),\delta_{1};\delta_{2}]= \mathcal{R}^{U}[\mathcal{R}^{U}[\phi(s),\delta_{2}],\delta_{1}];$
%  \item  $\mathcal{R}^{U}[\phi(s),\delta_{1}|\delta_{2}]= \mathcal{R}^{U}[\phi(s),\delta_{1}]\land \mathcal{R}^{U}[\phi(s),\delta_{2}];$
%  \item  $\mathcal{R}^{U}[\phi(s),(\pi x)\delta(x)]= (\forall x)\mathcal{R}^{U}[\phi(s),\delta(x)].$
%\end{itemize}
%\end{definition}

\begin{prop} \label{prop_Uregression}
Given a basic action theory $\mathcal{D}$, a Golog program $\delta$ and a situation-suppressed formula $\phi$, we have:
\begin{center}
  $\mathcal{D} \models \mathcal{R}^{U}[\phi(s),\delta]\equiv \forall s'. [Do(\delta, s, s')\supset \phi(s')].$
\end{center}
\end{prop}

\subsection{Generalized Planning Abstraction Framework}

Situation calculus cannot express the property of termination, counting, transitive closure, and non-deterministic actions. \citeauthor{Cui21} \shortcite{Cui21} extended the situation calculus for these four aspects: to represent the property of termination, following \cite{2004Representing}, they use the situation calculus with infinite histories; to represent planning  with non-deterministic actions, they treat a non-deterministic action as a non-deterministic program; to extended the situation calculus with counting, following the logic FOCN \cite{Kuske2017First}, they introduce counting terms of the form $\# \bar{y}.\varphi$, meaning the number of tuples $\bar{y}$ satisfying formula $\varphi$; transitive closure is often used to define counting terms, following transitive closure logic \cite{transclosure}, they introduced the notation $[TC_{\bar{x},\bar{y}} \varphi](\bar{u},\bar{v})$, where $\varphi(\bar{x},\bar{y})$ is a formula with $2k$ free variables, $\bar{u}$ and $\bar{v}$ are two $k$-tuples of terms, which says that the pair $(\bar{u},\bar{v})$ is contained in the reflexive transitive closure of the binary relation on $k$-tuples that is defined by $\varphi$.%, in case $P(x,y)$ is a binary predicate, we simply write $P^*(x,y)$ to mean $[TC_{x,y} P(x,y)](x,y)$.

In the following, we introduce the generalized planning abstraction framework proposed in \cite{Cui21}. %Firstly, we present the definition of generalized planning problems:

\begin{definition} \rm
A generalized planning problem is a triple $\mathcal{G}=\langle \mathcal{D},C,G \rangle$, where $\mathcal{D}$ is a BAT, $C$ is a trajectory constraint, i.e., a situation calculus formula with a free variable of infinite histories, and $G$ is a goal condition.
\end{definition}

 In the presence of non-deterministic actions, solutions to planning problems are programs whose execution under certain trajectory constraints are guaranteed to terminate and achieve the goal.

\begin{example}\rm \label{eg-clearing a block}
   In the blocks world, an agent can perform two kinds of actions: $mt(x)$ (move $x$ to the table, provided $x$ is being held, and $unstack(x,y)$ (unstack $x$ from $y$, provided $x$ is clear and $x$ is on $y$ ).There are four fluents: $clear(x)$, $ontable(x)$, $on(x,y)$, and $holding(x)$. In this problem, the trajectory constraint $C$ is $\top$, the goal state $G$ is $clear(A)$, some example axioms from $\mathcal{D}$ are as follows:\\
  \textbf{Precondition Axioms}:\\
  \vspace{1mm}
  \hspace*{0em}$Poss(mt(x),s) \equiv clear(x,s) \wedge \neg ontable(x,s) $.\\
  \textbf{Successor State Axioms}:\\
  \vspace{1mm}
  \hspace*{0em} $on(x,y,do(a,s)) \equiv  on(x,y,s) \wedge \neg a=unstack(x, y)$.\\
  \textbf{Initial Situation Axiom}:\\
  \hspace*{0.55em}$\exists x.on^{+}(x,A)\land ontable(A)\land \neg holding(x),$ where the formula $on^{+}(x,A)$ is a transitive closure formula with a concise form, which means that the block $x$ is above the block $A$.

\end{example}

Abstractions for generalized planning problems are specified by the following notion of refinement mapping:

\begin{definition}\rm \label{Ext-Refinement-Mapping}
A function $m$ is a refinement mapping from $\mathcal{G}_{h}\amendb{=\langle \mathcal{D}_{h},C_{h},G_{h} \rangle}$ to $\mathcal{G}_{l}\amendb{=\langle \mathcal{D}_{l},C_{l},G_{l} \rangle}$ if
for each HL \amendb{deterministic or non-deterministic} action type $A$, $m(A(\vec{x}))=\delta_{A}(\vec{x})$, where $\delta_{A}(\vec{x})$ is a LL program;
for each HL relational fluent $P$, $m(P(\vec{x}))=\phi_{P}(\vec{x})$, where $\phi_{P}(\vec{x})$ is a LL situation-suppressed formula;
for each HL functional fluent $f$, $m(f(\vec{x}))=\tau_f(\vec{x})$, where $\tau_f(\vec{x})$ is a LL (counting) term.
\end{definition}

%To relate a HL situation $s_{h}$ and a LL situation $s_{l}$, they introduce the notion of $m$-isomorphism relation, written $s_{h}\sim_{m}s_{l}$, as follows:
%$s_{h}\sim^{M_{h},M_{l}}_{m}s_{l}$
%Intuitively, for a HL situation $s_{h}$ and a LL situation $s_{l}$, if $s_{h}$ is $m$-isomorphic to $s_{l}$, then $s_{h}$ interprets HL symbols, iff $s_{l}$ interprets their refinements.
Given a refinement mapping $m$, they introduced an isomorphism relation, called $m\text{-}isomorphic$, between a HL and a LL situation as follows:
\begin{definition}\rm\label{df-homomorphic}
Given a refinement mapping $m$, a situation $s_{h}$ of a HL model $M_{h}$ is $m$-isomorphic to a situation $s_{l}$ in a LL model $M_{l}$, written $s_{h}\sim_{m}s_{l}$, if: for any HL relational fluent $P$, and variable assignment $v$, we have
      $M_{h}, v[s/s_{h}] \models P(\vec{x},s)$ iff $M_{l}, v[ s/s_{l}] \models m(P)(\vec{x}, s)$;
for any HL functional fluent $f$, variable assignment $v$, we have
      $M_{h}, v[s/s_{h}] \models f(\vec{x},s)=y$ iff $M_{l}, v[s/s_{l}] \models m(f)(\vec{x},s)=y$.
\end{definition}

To relate HL and LL models, they defined two relations: $m$-simulation and $m$-back-simulation. Intuitively, simulation means: whenever a refinement of a HL action can occur, so can the HL action, and back-simulation means the other direction. Here we only present the definition of $m$-simulation relation, $m$-back-simulation relation can be defined symmetrically. In the following, $\Delta^{M}_S$ denotes the situation domain of $M$; $S_0^{M}$ stands for the initial situation of $M$; the notation $\textit{Term}(\delta,s,C)$ means starting in situation $s$, program $\delta$ terminates under constraint $C$:
\begin{equation*}
	\setlength{\abovedisplayskip}{1mm}
	\setlength{\belowdisplayskip}{1mm}
    \hspace{0.35em}\textit{Term}(\delta,s,C)\doteq \neg \exists h.C(h)\land \forall s' \sqsubset h\exists \delta'.Trans^{*}(\delta, s, \delta', s').
\end{equation*}

%the abbreviation $\mbox{Term}(\delta,s,C)$ means starting in situation $s$, program $\delta$ terminates under constraint $C$.

\begin{definition}\rm
 A relation $B\subseteq \Delta^{M_{h}}_{S}\times \Delta^{M_{l}}_{S}$\amendb{ is an} $m$-simulation relation between $M_{h}$ and $M_{l}$, if $\langle S^{M_{h}}_{0}, S^{M_{l}}_{0} \rangle \in B$ and the following hold: (1) $\langle s_{h}, s_{l}\rangle \in B$ implies that: $s_{h}\sim_{m}s_{l}$; for any HL action type $A$, and variable assignment $v$, $M_{l},v[s/s_{l}]\models \mbox{Term}(m(A(\vec{x})),s,C_l)$, and if there is a situation $s'_{l}$ s.t. $M_{l},v[s/s_{l},s'/s'_{l}]\models Do(m(A(\vec{x})),s,s')$, then there is a situation $s'_{h}$ s.t. $M_{h},v[s/s_{h},s'/s'_{h}]\models  Do(A(\vec{x}),s,s')$ and $\langle s'_{h},s'_{l}\rangle \in B$. (2) For any infinite HL action sequence $\sigma$, if there is an infinite history in $M_l$ generated by $m(\sigma)$ and satisfying $C_l$, then there is an infinite history in $M_h$ generated by $\sigma$ and satisfying $C_h$.
\end{definition}

%\begin{definition}\rm ($m$-back-simulation)
%A relation $B\subseteq \Delta^{M_{h}}_{S}\times \Delta^{M_{l}}_{S}$ is an $m$-back-simulation relation between $M_{h}$ and $M_{l}$, if $\langle S^{M_{h}}_{0}, S^{M_{l}}_{0} \rangle \in B$, and the following hold:
%\begin{enumerate}
%\item $\langle s_{h}, s_{l}\rangle \in B$ implies that: $s_{h}\sim^{M_{h},M_{l}}_{m}s_{l}$;
% for any HL action $A$, and variable assignment $v$,
% $M_{l},v[s/s_{l}]\models \mbox{Term}(m(A(\vec{x})),s,C_l)$, and
% if there is a situation $s'_{h}$ s.t. $M_{h},v[s/s_{h},s'/s'_{h}]\models  Do(A(\vec{x}),s,s')$, then
% there is a situation $s'_{l}$ s.t. $M_{l},v[s/s_{l},s'/s'_{l}]\models Do(m(A(\vec{x})),s,s')$ and $\langle s'_{h},s'_{l}\rangle \in B$;
%\item For any infinite HL action sequence $\sigma$, if there is an infinite history in $M_h$ generated by $\sigma$ and satisfying $C_h$, then
%    there is an infinite history in $M_l$ generated by $m(\sigma)$ and satisfying $C_l$.
%\end{enumerate}
%\end{definition}

Based on the notions above, they defined sound/complete abstraction on model and theory level. For these two levels, sound abstractions mean that HL behavior entails LL behavior, complete abstractions mean the other direction.

\begin{definition}\rm \label{model-level-sound-abstraction}
$M_{h}$ is a {\em sound $m$-abstraction} of $M_{l}$, if
there is an $m$-back-simulation relation $B$ between $M_{h}$ and $M_{l}$.
\end{definition}

\begin{definition} \rm \label{model-level-complete-abstraction}
$M_{h}$ is a {\em complete $m$-abstraction} of $M_{l}$,  if there is a $m$-simulation relation $B$ between $M_{h}$ and $M_{l}$.
\end{definition}

On the theory level, for sound abstraction, they defined two versions (complete abstractions can be defined symmetrically, and we omit here):
\begin{definition} \rm \label{weak-sound}
 $\mathcal{G}_{h}$ is a weak sound $m$-abstraction of $\mathcal{G}_{l}$, if for any model $M_{l}$ of $\mathcal{D}_{l}$, there is a model $M_{h}$ of $\mathcal{D}_{h}$ such that: (1) $M_{h}$ is a sound $m$-abstraction of $M_{l}$ via $B$; (2) for any situations \amendc{$s_h$ in $M_{h}$} and $s_l$ in $M_{l}$, if $\langle s_{h}, s_{l}\rangle \in B$  and $M_h,v[s_h/s] \models G_h[s]$, then  $M_l,v[s_l/s] \models G_l[s]$.
\end{definition}

\begin{definition} \rm \label{soundabs_TL}
 $\mathcal{G}_{h}$ is a sound $m$-abstraction of $\mathcal{G}_{l}$, if it is a weak sound $m$-abstraction of $\mathcal{G}_{l}$, and
  for any model $M_{l}$ of $\mathcal{D}_{l}$, there is a model $M_{h}$ of $\mathcal{D}_{h}$ such that: (1) $M_{h}$ is a complete $m$-abstraction of $M_{l}$ via $B$; (2) for any situations \amendc{$s_h$ in $M_{h}$} and $s_l$ in $M_{l}$, if $\langle s_{h}, s_{l}\rangle \in B$  and $M_h,v[s_h/s] \models G_h[s]$, then  $M_l,v[s_l/s] \models G_l[s]$.
\end{definition}

\section{Proof-Theoretic Characterization}

In this section, we give a proof-theoretic characterization for sound abstractions for generalized planning (g-planning).

First of all, we introduce some notations and conventions. We define the program of doing any HL action sequence
and its refinement as follows:
\begin{equation*}
	\setlength{\abovedisplayskip}{1mm}
	\setlength{\belowdisplayskip}{1mm}
    anyhlas\doteq (|_{A\in \mathcal{A}_{h} }\ \pi \vec{x}.A(\vec{x}))^{*}, \thinspace anyllps\doteq m(anyhlas).
\end{equation*}
We call a situation $s$ s.t. $Do(anyllps, S_0, s)$ holds an executable refinement of a HL situation.

The notation $\textit{Infexe}(\delta, h, C)$ means $h$ is an infinite execution of program $\delta$ satisfying trajectory constraint $C$:
\begin{equation*}
	\setlength{\abovedisplayskip}{1mm}
	\setlength{\belowdisplayskip}{1mm}
    \textit{Infexe}(\delta,h,C)\doteq C(h)\land \forall s' \sqsubset h\exists \delta'.\ Trans^{*}(\delta, S_0, \delta', s').
\end{equation*}

% The notation $\textit{Term}(\delta,s,C)$ means starting in situation $s$, program $\delta$ terminates under constraint $C$:
% \begin{equation*}
% 	\setlength{\abovedisplayskip}{1mm}
% 	\setlength{\belowdisplayskip}{1mm}
%     \hspace{0.35em}\textit{Term}(\delta,s,C)\doteq \neg \exists h.C(h)\land \forall s' \sqsubset h\exists \delta'.Trans^{*}(\delta, s, \delta', s').
% \end{equation*}

We introduce an abbreviation $R(s,s')$, which means that situations $s$ and $s'$ result from executing the refinement of the same HL action sequence:
\begingroup
\addtolength{\jot}{-1mm}
\begin{flalign*}
	\setlength{\abovedisplayskip}{1mm}
	\setlength{\belowdisplayskip}{1mm}
    &\hspace{0.35em} R(s,s')\doteq \forall P. P(S_{0},S_{0})\land \\
	&\hspace*{1.35em} \mathsmaller{\bigwedge}_{A\in \mathcal{A}_{h}} \forall \vec{x},s,s_{1},s',s'_{1}. (P(s,s')\land Do(m(A(\vec{x})),s,s_{1}) \\
	&\hspace*{4.35em}\land Do(m(A(\vec{x})),s',s'_{1})\supset P(s_{1},s'_{1}))\supset P(s,s').
\end{flalign*}
\endgroup

Let $\phi$ be a HL formula uniform in a situation. We use $m(\phi)$ to denote the formula resulting from replacing each high-level symbol in $\phi$ with its LL definitions. We now define $m(C)$ for a HL trajectory constraint $C$. For this purpose, we first define a normal form for trajectory constraints.

\begin{definition} \rm
Let $C$ be a trajectory constraint. We say that $C$ is in normal form if $C$ contains no occurrence of action variables or $Poss$, and any appearance of $do$ must be in the form of $s'=do(A(\vec{t}),s)$, where $s$ and $s'$ are variables.
\end{definition}

It is easy to prove that any trajectory constraint can be converted to an equivalent one in normal form.

\begin{definition} \rm \label{refinement_of_TC}
Let $C$ be a HL trajectory constraint. We first convert it into an equivalent one $C'$ in normal form. We let $m(C)$ denote the LL constraint obtained from $C'$ as follows: first replace any appearance of $\exists s$ with $\exists s. Do(anyllps, S_0, s)$, then replace any appearance of $s'=do(A(\vec{t}),s)$ with $Do(m(A(\vec{t})), s, s')$, and finally replace any high-level symbols with its LL definitions.
\end{definition}

We now extend the $m$-simulation or $m$-back-simulation relation $B$ to infinite histories, and show that if a HL infinite history $h_h$ and a LL infinite history $h_l$ are $B$-related, then $h_h$ satisfies a constraint $C$ iff $h_l$ satisfies $m(C)$.

%\begin{definition} \rm
%Let $M_h$ be a sound or complete abstraction of $M_l$ via $B$.
%Let $h_h$ and $h_l$ be infinite histories of $M_h$ and $M_l$, respectively. We write $\langle h_{h}, h_{l}\rangle \in B$ if for any $s_h \sqsubset h_h$ and any $s_l \sqsubset h_l$, there exist $s'_h \sqsubset h_h$ and $s'_l \sqsubset h_l$, such that $s_h \sqsubseteq s'_h$, $s_l \sqsubseteq s'_l$, and $\langle s'_{h}, s'_{l}\rangle \in B$.
%\end{definition}

\begin{definition} \rm
Let $M_h$ be a sound or complete abstraction of $M_l$ via $B$.
Let $h_h$ and $h_l$ be infinite histories of $M_h$ and $M_l$, respectively. We write $\langle h_{h}, h_{l}\rangle \in B$, if for any $s_h \sqsubset h_h$, there exist $s_l \sqsubset h_l$, such that $\langle s_{h}, s_{l}\rangle \in B$.
\end{definition}

By induction on the structure of a normal form trajectory constraint, we can prove:
%\begin{prop}
%Let $M_h$ be a sound or complete abstraction of $M_l$ via $B$. Let $\langle h_{h}, h_{l}\rangle \in B$. Let $C$ be a HL trajectory constraint. Then
% $M_{h}, v[h/h_{h}] \models C(h)$ iff $M_{l}, v[h/h_{l}] \models m(C)(h)$.
%\end{prop}
%

\begin{prop}\label{simulation between infinite histories}
Let $M_h$ be a sound and complete abstraction of $M_l$. Let $C$ be a HL trajectory constraint. Then
 $M_{h}\models C$ iff $M_{l}\models m(C)$.
\end{prop}

\begin{proof}
  We do the proof by using structural induction on $C$:

\vspace{2mm}
  \noindent \textbf{Induction base}:
  \begin{itemize}
    \item $C\doteq s=s'$: Let $s_{h}$ and $s'_{h}$ are two situations in $M_{h}$, and $C$ be the atom $s_{h}=s'_{h}$. On one hand, if we have $M_{h},v[s/s_{h},s'/s'_{h}]\models s=s'$, then there  exists a HL grounded action sequence $\alpha$, such that
          $M_{h}, v[s/s_{h},s'/s'_{h}]\models s=do(\alpha,S_{0})\land s'=do(\alpha,S_{0})$.  Based on the definition of sound abstraction on model level, we have that there exist two situations $s_{l}$ and $s'_{l}$ such that $M_{l}, v[s/s_{l},s/s'_{l}]\models Do(m(\alpha),S_{0},s)\land Do(m(\alpha),S_{0},s')$.  Given the Definition \ref{refinement_of_TC}, we can get $M_{l}\models m(s_{h}=s'_{h})$.

        On the other hand, if we have $M_{l}\models m(s_{h}=s'_{h})$, then for the HL action sequence $\alpha$, based on the Definition \ref{refinement_of_TC}, we have     $M_{l}, v[s/s_{l},s/s'_{l}]\models Do(m(\alpha),S_{0},s)\land Do(m(\alpha),S_{0},s')$.
        By using the definition of complete abstraction on model level, we have that there exist two situations $s_{h}$ and $s'_{h}$ such that $M_{h}, v[s/s_{h},s'/s'_{h}]\models s=do(\alpha,S_{0})\land s'=do(\alpha,S_{0})$. Then we can get $M_{h}\models s_{h}=s'_{h}$.

    \item $C\doteq s\sqsubseteq s'$:  Let $s_{h}$ and $s'_{h}$ are two situations in $M_{h}$, and $C$ be the atom $s_{h}\sqsubseteq s'_{h}$. On one hand, there exists two HL grounded action sequences $\alpha$ and $\alpha'$, such that $\alpha$ is a subsequence of $\alpha'$, and
          $M_{h},v[s/s_{h},s/s'_{h}] \models s=do(\alpha, S_{0})\land s'=do(\alpha',S_{0})\land s\sqsubseteq s'$.
        Based on the definition of sound abstraction on model level, we have that there exist two situations $s_{l}$ and $s'_{l}$ such that
          $M_{l}, v[s/s_{l},s/s'_{l}]\models Do(m(\alpha),S_{0},s)\land Do(m(\alpha'),S_{0},s')\land s\sqsubseteq s'$.
        Given the Definition \ref{refinement_of_TC}, we can get $M_{l}\models m(s_{h}\sqsubseteq s'_{h})$.

        On the other hand, if we have $M_{l}\models m(s_{h}\sqsubseteq s'_{h})$, then for the HL action sequences $\alpha$ and $\alpha'$, we can get
          $M_{l}, v[s/s_{l},s/s'_{l}]\models Do(m(\alpha),S_{0},s)\land Do(m(\alpha'),S_{0},s')\land s\sqsubseteq s'$.
         Based on the Definition \ref{refinement_of_TC}, we have
           $M_{l}, v[s/s_{l},s/s'_{l}]\models Do(m(\alpha),S_{0},s)\land Do(m(\alpha),S_{0},s')\land s\sqsubseteq s'$.
        By using the definition of complete abstraction on model level, we have that there exist two situations $s_{h}$ and $s'_{h}$ such that
          $M_{h}, v[s/s_{h},s'/s'_{h}]\models s=do(\alpha,S_{0})\land s'=do(\alpha,S_{0}) \land s\sqsubseteq s'$.
        Then we can get $M_{h}\models s_{h}\sqsubseteq s'_{h}$.

    \item $C\doteq s\sqsubset h$:    Let $s_{h}$ be a situation in $M_{h}$, $h_{h}$ be a infinite history in $M_{h}$, and $C$ be the atom $s_{h}\sqsubset h_{h}$. On one hand, there exists an infinite HL grounded action sequence $\sigma$ and a finite subsequence $\alpha$ of $\sigma$, such that
      $M_{h},v[s/s_{h},h/h_{h}] \models s=do(\alpha,S_{0})\land s\sqsubset h$.
    Using the action sequence $\sigma$, and based on the definition of sound abstraction on model level, we can easily construct an infinite history $h_{l}$ in $M_{l}$ corresponds to $h_{h}$, and there also exists a situation $s_{l}\sqsubset h_{l}$ such that
      $M_{l},v[s/s_{l},h/h_{l}] \models Do(m(\alpha),S_{0},s)\land s\sqsubset h$.
    Given the Definition \ref{refinement_of_TC}, we can get $M_{l}\models m(s_{h}\sqsubset h_{h})$.

    On the other hand, if we have $M_{l}\models m(s_{h}\sqsubset h_{h})$, then for the HL action sequences $\sigma$ and $\alpha$, based on the Definition \ref{refinement_of_TC}, we have
      $M_{l},v[s/s_{l},h/h_{l}] \models Do(m(\alpha),S_{0},s)\land s\sqsubset h$.
    By using the definition of complete abstraction on model level, we have that there exist a situation $s_{h}$ and an infinite history  $h_{h}$ generated by $\sigma$ such that
      $M_{h},v[s/s_{h},h/h_{h}] \models s=do(\alpha,S_{0})\land s\sqsubset h$.
    Then we can get $M_{h}\models s_{h}\sqsubset h_{h}$.

    \item $C\doteq a=a'$:    Let $a$ and $a'$ are two HL primitive deterministic actions, and $C$ be the atom $a=a'$. Then based on the definition of refinement mapping, we have that
      $M_{h} \models a=a' \Leftrightarrow M_{l}\models m(a)=m(a')\Leftrightarrow M_{l}\models m(a = a')$.

    \item $C\doteq P(\vec{x},s)$:    Let $s_{h}$ is a situation in $M_{h}$, $P$ is a HL relational fluent, and $C$ be the atom $P(\vec{x},s)$. Then based on the definition of $m$-isomorphic, we have that there exists a situation $s_{l}$ in $M_{l}$ such that
      $M_{h},v[s/s_{h}] \models P(\vec{x},s) \Leftrightarrow  M_{l},v[s/s_{l}]\models m(P)(\vec{x},s)$.

    \item $C\doteq f(\vec{x},s)=y$:    Let $s_{h}$ is a situation in $M_{h}$, $f$ is a HL relational fluent, and $C$ be the atom $f(\vec{x},s)=y$. Then based on the definitions of simulation relations and $m$-isomorphic, we have that there exists a situation $s_{l}$ in $M_{l}$ such that
      $M_{h},v[s/s_{h}] \models f(\vec{x},s)=y \Leftrightarrow  M_{l},v[s/s_{l}]\models m(f)(\vec{x}, s)=y$.
  \end{itemize}

\vspace{2mm}
  \noindent \textbf{Induction step}: Let $C_{1}$ and $C_{2}$ are two HL trajectory constraints, and suppose that
  \begin{center}
    $M_{h}\models C_{1}$ iff $M_{l}\models m(C_{1})$,
    $M_{h}\models C_{2}$ iff $M_{l}\models m(C_{2})$.
  \end{center}
  Then we prove the following three cases:

  \begin{itemize}
    \item $C\doteq \neg C_{1}$:    If we have $M_{h}\models \neg C_{1}$, then we can get $M_{h}\nvDash C_{1}$. Based on the inductive assumption, we can get $M_{l}\nvDash m(C_{1})$, which means $M_{l}\models \neg m( C_{1})$. Thus, we have $M_{l}\models m(\neg C_{1})$.

    \item $C\doteq C_{1}\land C_{2}$:    If we have $M_{h}\models C_{1}\land C_{2}$, then we can get $M_{h}\models C_{1}$ and $M_{h}\models C_{2}$. Based on the inductive assumption, we can get $M_{l}\models m(C_{1})$ and $M_{l}\models m(C_{2})$, which means $M_{l}\models m(C_{1})\land m(C_{2})$. Thus, we have $M_{l}\models m(C_{1}\land C_{2})$.
    \item $C\doteq \exists s. C_{1}(s)$:    If we have $M_{h}\models \exists s. C_{1}(s)$, then there exists a situation $s_{h}$ in $M_{h}$ reached from the initial situation $S_{0}$ via executing an action sequence $\alpha$, such that $M_{h}, v[s/s_{h}]\models C_{1}(s)$. Based on the definition of sound abstraction on model level and inductive assumption, we have that there exists a situation $s_{l}$ in $M_{l}$ such that $M_{l},v[s/s_{l}]\models Do(m(\alpha),S_{0},s)\land m(C_{1})(s)$, which means $M_{l}\models m(\exists s. C_{1}(s))$.
  \end{itemize}

\end{proof}

%?? please use spaces to let the formulas look nicer

Non-deterministic actions in \cite{Cui21} are treated as non-deterministic programs. In particular, each non-deterministic action $A$ has a definition in the form $A(\vec{x})\doteq \pi \vec{u}. A_d(\vec{x},\vec{u})$, where $A_d$ is a deterministic action.  %Recall that for each deterministic action $A_d$, its action precondition axiom has form $Poss(A(\vec{x},\vec{u}),s)\equiv \Pi_A(\vec{x},\vec{u},s)$; for each relational fluent $P$, its SSA has form $P(\vec{y},do(a,s))\equiv \phi_{P}(\vec{y},a,s)$; and  for each functional fluent $f$, its SSA has form  $f(\vec{y},do(a,s))=z\equiv \psi_{f}(\vec{y},z,a,s)$.
We let $\Pi_A(\vec{x},s)$ denote $\exists \vec{u}. \Pi_{A_d}(\vec{x},\vec{u},s)$, let
$\phi_{P,A_d}(\vec{y},\vec{x},\vec{u}, s)$ denote $\phi_P(\vec{y},A_d(\vec{x},\vec{u}),s)$ simplified by using $\mathcal{D}_{una}$, and let
$\psi_{f,A_{d}}(\vec{y},z,\vec{x},\vec{u},s))$  denote $\psi_f(\vec{y},z,A_d(\vec{x},\vec{u}),s)$ simplified by using $\mathcal{D}_{una}$.

We now introduce the following abbreviations:

\vspace*{1mm}
\noindent$\psi_T \doteq \bigwedge_{A\in \mathcal{A}_{h}}\forall \vec{x}. \textit{Term}(m(A(\vec{x})),s,C_l)$,

\vspace*{1mm}
\noindent$\xi_P \doteq \bigwedge_{A\in \mathcal{A}_{h}}\forall \vec{x},s'. Do(m(A(\vec{x})),s,s') \supset \exists \vec{u}.$\\
\hspace*{1em}$\bigwedge_{P\in \mathcal{P}_{h}}[\forall \vec{y}. m(P(\vec{y},s'))\equiv m(\phi_{P,A_{d}}(\vec{y},\vec{x},\vec{u}, s))]$,

\vspace*{1mm}
\noindent$\xi_f \doteq \bigwedge_{A\in \mathcal{A}_{h}}\forall \vec{x},s'. Do(m(A(\vec{x})),s,s') \supset \exists \vec{u}.$\\
\hspace*{1em}$\bigwedge_{f\in \mathcal{F}_{h}}[\forall \vec{y}, z. m(f(\vec{y},s')=z)\equiv m(\psi_{f,A_{d}}(\vec{y},z,\vec{x},\vec{u},s))]$,

\vspace*{1mm}
\noindent where, $\psi_T$ says that the refinement of any HL action terminates in $s$ under $C_l$; $\xi_P$ and $\xi_f$ says that for any HL action $A(\vec{x})$, if its refinement transforms situation $s$ to $s'$, then there is $\vec{u}$ s.t. the mapping of all SSAs instantiated with $A_d(\vec{x},\vec{u})$ hold for $s$ and $s'$.

The following theorem gives a proof-theoretic characterization for sound abstractions, where Item 1 and Item 6 are easy to understand; Item 2 says that $\mathcal{D}_{l}$ entails that for any executable refinement of a HL situation,  the executability of the refinement of any HL action implies its mapped precondition; Item 3 says that $\mathcal{D}_{l}$ entails  that for any executable refinement of a HL situation, the mapped precondition of any HL action implies that the executability of its refinement holds in some $R$-related situation; and Item 5 says that $\mathcal{D}_{l}$ entails  that the existence of an inifinite execution of $anyllps$ satisfing the LL constraint is equivalent to the existence of one satisfying the mapped HL constraint.

\begin{theorem}(Sound abstraction) \label{s-abs-nd-case}
  Given a generalized planning problem $\mathcal{G}_{l}$ and its abstraction $\mathcal{G}_{h}$, $\mathcal{G}_{h}$ is a sound $m$-abstraction of $\mathcal{G}_{l}$ iff the following hold:
  \begin{enumerate}
  	\item $\mathcal{D}_{l}^{S_{0}}\models m(\phi)$, where $\phi\in \mathcal{D}_{h}^{S_{0}}$;
	\item $\mathcal{D}_{l}\models \forall s. Do(anyllps,S_{0},s)\supset \\
	\hspace*{1em}\bigwedge_{A\in \mathcal{A}_{h}}\forall \vec{x},s'. Do(m(A(\vec{x})),s,s')\supset m(\Pi_A(\vec{x},s))$;

	\item $\mathcal{D}_{l}\models \forall s. Do(anyllps,S_{0},s)\supset \\
	\hspace*{1em}\bigwedge_{A\in \mathcal{A}_{h}}\forall \vec{x}. m(\Pi_A(\vec{x},s))\supset\\
	\hspace*{2em} \exists s',s''.R(s,s')\land Do(m(A(\vec{x})),s',s'')$;

    \item $\mathcal{D}^{l}\models \forall s. Do(anyllps,S_{0},s)\supset \\
    \hspace*{1em}\bigwedge_{A\in \mathcal{A}_{h}}\forall \vec{x},s'. Do(m(A(\vec{x})),s,s') \supset \psi_T$; % \bigwedge_{A\in \mathcal{A}_{h}}\forall \vec{x}. \textit{Term}(m(A(\vec{x})),s,C_l)$;

    \item $\mathcal{D}_{l}\models \forall s. Do(anyllps,S_{0},s)\supset \xi_P$; % \bigwedge_{A\in \mathcal{A}_{h}}\forall \vec{x},s'. Do(m(A(\vec{x})),s,s') $\\ \hspace*{2.5em}$ \supset\exists\vec{u}.\bigwedge_{P\in \mathcal{P}_{h}}[\forall \vec{y}. m(P(\vec{y},s'))\equiv m(\phi_{P,A_{d}}(\vec{y},\vec{x},\vec{u}, s))]$;

    \item $\mathcal{D}_{l}\models \forall s. Do(anyllps,S_{0},s)\supset \xi_f$; %\bigwedge_{A\in \mathcal{A}_{h}}\forall \vec{x},s'. Do(m(A(\vec{x})),s,s')$\\\hspace*{2.5em}$\supset \exists \vec{u}.\bigwedge_{f\in \mathcal{F}_{h}}[\forall \vec{y}, z. m(f(\vec{y},s')=z)\equiv m(\psi_{f,A_{d}}(\vec{y},z,\vec{x},\vec{u},s))]$,

	\item $\mathcal{D}_{l}\models \exists h_{l}.\textit{Infexe}(anyllps,h_{l},C_{l})\\
	\hspace*{1em}\equiv \exists h_{l}. \textit{Infexe}(anyllps,h_{l},m(C_{h}))$;

	\item $\mathcal{D}_{l}\models \forall s. Do(anyllps,S_{0},s)\land m(G_{h})[s]\supset G_{l}[s]$.
  \end{enumerate}
%  where, $poss(A(\vec{x}),s)\equiv \mathcal{R}^{E}[\top(s), A(\vec{x})]$, and $\phi^{*}_{P,A}(\vec{x},\vec{y},s) \equiv \mathcal{R}^{U}[P(\vec{y},s'), A(\vec{x})]$, and $\psi^{*}_{f,A}(\vec{x},\vec{y},s) \equiv \mathcal{R}^{U}[ f(\vec{y},s')=z, A(\vec{x})]$.
\end{theorem}

\begin{proof}
   Firstly, we prove the \textbf{only-if direction}:

\begin{enumerate}
  \item[1.] For any $M^{S_0}_{l}\models \mathcal{D}^{S_{0}}_{l}$, we can extend it to a model $M_{l}$ of $\mathcal{D}_{l}$. Then there exists a model $M_{h}$ of $\mathcal{D}_{h}$, which is a sound $m$-abstraction of $M_{l}$, and we have that the initial situation $S^{M_{h}}_{0}$ of $M_{h}$ is $m$-isomorphic to the initial situation $S^{M_{l}}_{0}$ of $M_{l}$. Given $\phi\in \mathcal{D}^{S_{0}}_{h}$, since $S^{M_{h}}_{0}$ satisfies $\phi$, so $S^{M_{l}}_{0}$ satisfies $m(\phi)$. Therefore, $\mathcal{D}^{S_{0}}_{l}\models m(\phi)$ for $\phi\in \mathcal{D}^{S_{0}}_{h}$;

\end{enumerate}

\begin{enumerate}
  \item[2.] Let $M_{l}$ be a model of $\mathcal{D}_{l}$, then there exists a model $M_{h}$ of $\mathcal{D}_{h}$, which is a complete $m$-abstraction of $M_{l}$ via a $m$-simulation relation $B_{1}$. Let $s_{l}$ be a situation of $M_{l}$ which satisfies $Do(anyllps,  S_{0},s)$, then based on the definition of $m$-simulation relation, we have that there is a situation $s_{h}$ of $M_{h}$ such that $\langle s_{h},s_{l} \rangle \in B_{1}$. Thus, for any HL action $A(\vec{x})$, $s_{h}$ satisfies $\forall \vec{x}.\Pi_{A}(A(\vec{x}),s)$ iff $s_{l}$ satisfied $\forall \vec{x}. m(\Pi_{A}(A(\vec{x}),s))$. Furthermore, if $M_{l},v[s/s_{l}]\models \forall \vec{x},\exists s'_{l}. Do(m(A(\vec{x})),s,s'_{l})$, then we have $M_{h},v[s/s_{h}]\models \forall \vec{x},\exists s'_{h}. Do(A(\vec{x}),s,s'_{h})$, which implies $M_{h},v[s/s_{h}]\models \forall \vec{x}. \Pi_{A}(A(\vec{x}),s)$. Therefore, $M_{l},v[s/s_{l}]\models \forall \vec{x}. m(\Pi_A(\vec{x},s))$.
\end{enumerate}

\begin{enumerate}
  \item[3.] Let $M_{l}$ be a model of $D_{l}$, then there exists a model $M_{h}$ of $D_{h}$ such that $M_{h}$ is a complete abstraction for $M_{l}$ via a $m$-simulation relation $B_{1}$, and sound abstraction for $M_{l}$ via a $m$-back-simulation relation $B_{2}$. Let $s_{l}$ be a situation of $M_{l}$ which satisfies $Do(anyllps, S_{0},s)$, then based on the definitions of $m$-simulation and $m$-back-simulation relation, we have that there is a situation $s_{h}$ of $M_{h}$ such that $\langle s_{h},s_{l} \rangle \in B_{1}$, and a situation $s'_{l}$ of $M_{l}$ such that $\langle s_{h},s'_{l} \rangle \in B_{2}$. Thus, $s_{l}$ and $s'_{l}$ are $R$-related. Furthermore, for any HL action $A(\vec{x})$, if $M_{l},v[s/s_{l}]\models \forall \vec{x}. m(\Pi_{A}(A(\vec{x}),s))$, then $M_{h},v[s/s_{h}]\models \forall \vec{x}. \Pi_{A}(A(\vec{x}),s)$, and hence $M_{l},v[s/s'_{l}]\models \forall \vec{x}. m(\Pi_{A}(A(\vec{x}),s))$. Therefore, we have $M_{l},v[s/s_{l}]\models \forall \vec{x}, \exists s', s''.R(s,s') \land Do(m(A(\vec{x})),s',s'')$.

\end{enumerate}

\begin{enumerate}
  \item[4.] Let $M_{l}$ be a model of $D_{l}$, then there exists a model $M_{h}$ of $D_{h}$ such that $M_{h}$ is a complete abstraction for $M_{l}$ via a $m$-simulation relation $B_{1}$. Let $s_{l}$ be a situation of $M_{l}$ which satisfies $Do(anyllps, S_{0},s)$, then  base on the definition of $m$-simulation relation, we have that there is a situation $s_{h}$ of $M_{h}$ such that $\langle s_{h},s_{l} \rangle \in B_{1}$, and $M_{l},v[s/s_{l}]\models \forall \vec{x}. Term(A(\vec{x}),s,C_{l})$ for any high HL action $A(\vec{x})$.

\end{enumerate}

\begin{enumerate}
  \item[5.] Let $M_{l}$ be a model of $\mathcal{D}_{l}$, then there exists a model $M_{h}$ of $\mathcal{D}_{h}$, which is a complete $m$-abstraction of $M_{l}$ via a $m$-simulation relation $B_{1}$. Let $s_{l}$ be a situation of $M_{l}$ which satisfies $Do(anyllps, S_{0},s)$, then  base on the definition of $m$-simulation relation, we have that there is a situation $s_{h}$ of $M_{h}$, such that $\langle s_{h},s_{l} \rangle \in B_{1}$. For each HL action $A(\vec{x})$, and any situation $s'_{l}$, if $M_{l},v[s/s_{l}, s'/s'_{l}]\models \forall \vec{x}. Do(m(A(\vec{x})),s,s')$, then there exists an execution $A(\vec{x}, \vec{u})$ of $A(\vec{x})$, and a situation $s'_{h}$, such that $M_{h},v[s/s_{h},s'/s'_{h}]\models \forall \vec{x}, \exists \vec{u}. Do(A_{d}(\vec{x},\vec{u}),s,s')$, and $\langle s'_{h},s'_{l} \rangle \in B_{1}$. Then we have that $s_{l}$ satisfies $\forall \vec{x},\vec{y},\exists \vec{u}. m(\phi_{P,A_{d}}(\vec{x},\vec{y},\vec{u},s))$ iff $s_{h}$ satisfies $\forall \vec{x},\vec{y},\exists \vec{u}.\phi_{P,A_{d}}(\vec{x},\vec{y},\vec{u},s)$ iff $s'_{h}$ satisfies $\forall \vec{y}. P(\vec{y},s')$ iff $s'_{l}$ satisfies $\forall \vec{y}. m(P(\vec{y},s'))$.
\end{enumerate}

\begin{enumerate}
  \item[6.] The proof of item 6 is similar to item 5.
\end{enumerate}

\begin{enumerate}
  \item[7.] Let $M_{l}$ be a model of $\mathcal{D}_{l}$, then there exists a model $M_{h}$ of $D_{h}$ such that $M_{h}$ is a complete abstraction for $M_{l}$ via a $m$-simulation relation $B_{1}$, and sound abstraction for $M_{l}$ via a $m$-back-simulation relation $B_{2}$.
      \begin{itemize}
        \item[(i)] On one hand, Let $h_{l}$ be a LL infinite history, for an infinite ground HL action sequence
              $A_{1}, A_{2}, \cdots, A_{i},\cdots$,
            and let $\sigma \doteq A_{1}, A_{2}, \cdots, A_{i},$
            we consider the infinite LL situation sequence
              $S^{M_{l}}_{0}, s^{1}_{l}, s^{2}_{l}, \cdots, s^{i}_{l}, \cdots$,
            where $s^{i}_{l}$ satisfies the following formula
              $s^{i}_{l}\sqsubset h_{l}\land Do(m(\sigma), S_{0}, s^{i}_{l})$. Based on the $m$-simulation relation $B_{1}$, we can easily construct an HL infinite history $h_{h}$ such that $\langle h_{h},h_{l} \rangle \in B_{1}$. Furthermore, we can get that if $h_{l}$ satisfies   $M_{l},v[h/h_{l}]\models Infexe(anyllps, h_{l}, C_{l})$, then $h_{h}$ satisfies $M_{h},v[h/h_{h}]\models Infexe(anyllps, h_{h}, C_{h})$. Then, based on the Proposition \ref{simulation between infinite histories}, we can get that $M_{l}, [h/h_{l}]\models Infexe(anyllps,h_{l},m(C_{h}))$.
        \item[(ii)] On the other hand, given an LL infinite history $h_{l}$, we can also construct an HL infinite history $h_{h}$ based on the $m$-simulation relation $B_{1}$ such that such that $\langle h_{h},h_{l} \rangle \in B_{1}$.     If $h_{l}$ satisfies $M_{l}\models \exists h_{l}. Infexe(angllps, h_{l}, m(C_{h}))$, then according to the proposition \ref{simulation between infinite histories}, we can get that $M_{h},v[h/h_{h}] \models Infexe(anyllps, h_{h}, C_{h})$. Given the definition of $m$-back-simulation relation, we then have that there exists a low-level infinite history $h_{l}$ which satisfies $M_{l}, v[h/h_{l}]\models Infexe(anyllps, h_{l}, C_{l})$.
      \end{itemize}
\end{enumerate}

\begin{enumerate}
  \item[8.] Let $M_{l}$ be a model of $\mathcal{D}_{l}$, then there exists a model $M_{h}$ of $\mathcal{D}_{h}$, which is a complete $m$-abstraction of $M_{l}$ via a $m$-simulation relation $B_{1}$. Let $s_{l}$ be a situation of $M_{l}$ which satisfies $Do(anyllps, S_{0},s)$, then base on the definition of $m$-simulation relation, we have that there is a situation $s_{h}$ of $M_{h}$, such that $\langle s_{h},s_{l} \rangle \in B_{1}$. Suppose $s_{l}$ satisfies $m(G_{h})$, then $s_{h}$ satisfies $G_{h}$, according to the definition of weak sound abstraction on theory level, we can get that $s_{l}$ satisfies $G_{l}$.
\end{enumerate}

Secondly, we prove the \textbf{if direction}: For any model $M_{l}$ of $\mathcal{D}_{l}$, we construct a model $M_{h}$ of $\mathcal{D}_{h}$ as follows: $M_{h}$ interprets the HL relational and functional fluents at initial situation $S^{M_{h}}_{0}$ according to the initial situation $S^{M_{l}}_{0}$ of $M_{l}$ and the refinement mapping. We complete $M_{h}$ by using the action precondition and effect axioms.

For $M_{l}$ and $M_{h}$, we firstly prove that $M_{h}$ is a complete abstraction of $M_{l}$: In the following, we construct a $m$-simulation relation $B_{1}$ between the situation and infinite history domains of $M_{h}$ and $M_{l}$,
    \begin{itemize}
     \item[1.]   $\langle S^{M_{h}}_{0}, S^{M_{l}}_{0}\rangle \in  B_{1}$,  i.e., $S^{M_{h}}_{0}$ is $m$-isomorphic to $S^{M_{l}}_{0}$;
     \item[2.] Let $\langle s_{h},s_{l}\rangle \in B_{1}$, where $s_{l}$ is a situation reached from  $S^{M_{l}}_{0}$ via executing $anyllps$, then $s_{h}$ is $m$-isomorphic to $s_{l}$. Let $A(\vec{x})$ be an arbitrary HL action. %By item 2, we have that if $s_{l}$ satisfies $poss(m(A(\vec{x})),s_{l})$, then $s_{h}$ satisfies $\Pi_A(\vec{x},s_{h})$.
         By item 4, if $m(A(\vec{x}))$ is executable in the situation $s_{l}$, then we have $M_{l},v[s/s_{l}]\models \forall \vec{x}. Term(m(A(\vec{x})),s,C_{l})$. By item 2, if $m(A(\vec{x}))$ is executable in $s_{l}$, then $A(\vec{x})$ is executable in $s_{h}$. For each situation $s'_{l}$ satisfies $Do(m(A(\vec{x})),s_{l},s'_{l})$ and all the situations $s'_{h}$ that satisfy $Do(A(\vec{x}),s_{h},s'_{h})$, we add all the pairs $\langle s'_{h}, s'_{l}\rangle$ in $B_{1}$.  By item 5 (the case of item 6 can be discussed similarly), for an execution $A_{d}(\vec{x},\vec{u})$ of $A(\vec{x})$, $s'_{l}$ satisfies $\forall \vec{y}. m(P(\vec{y},s'_{l}))$ iff $s_{l}$ satisfies  $\forall \vec{x},\vec{y},\exists \vec{u}. m(\Pi_{A_{d}}(\vec{x},\vec{y},\vec{u},s_{l}))$ iff $s_{h}$ satisfies $\forall \vec{x},\vec{y},\exists \vec{u}.\Pi_{A_{d}}(\vec{x},\vec{y},\vec{u},s_{h})$. For the situations $s'_{h}$ that do not satisfy $\forall \vec{y}. P(\vec{y},s'_{h})$, we then delete all the pairs $\langle s'_{h},s'_{l}\rangle$ from $B_{1}$.

         %Then we delete all the  have that $s'_{l}$ satisfies $m(P(\vec{y},s'_{l}))$ iff $s_{l}$ satisfies $m(\Pi_{A_{d}}(\vec{x},\vec{y},\vec{u},s_{l}))$ iff iff $s'_{h}$ satisfies $P(\vec{y},s'_{h})$. Thus, we have that $s'_{h}$ is $m$-isomorphic to $s'_{l}$, and we add $\langle s'_{h}, s'_{l} \rangle $ in $B$.
     \item[3.]  $B_{1}$ is a $m$-simulation relation follows its construction coupled with item 7.

   \end{itemize}

Furthermore, if $\langle s_{h}, s_{l}\rangle \in B_{1}$ and $s_{h}$ satisfies $G_{h}$, then $s_{l}$ satisfies $G_{l}$. Since $s_{h}$ and $s_{l}$ are $m$-isomorphic, $s_{l}$ satisfies $m(G_{h})$, by Item 8, $s_{l}$ satisfies $G_{l}$. Therefore, $M_{h}$ is a sound abstraction of $M_{l}$.
\vspace{2mm}

For $M_{l}$ and $M_{h}$, we then prove that $M_{h}$ is a sound abstraction of $M_{l}$: In the following, we construct a $m$-back-simulation relation $B_{2}$ between the situation and infinite history domains of $M_{h}$ and $M_{l}$,
   \begin{itemize}
     \item[1.] $\langle S^{M_{h}}_{0}, S^{M_{l}}_{0}\rangle \in  B_{2}$,  i.e., $S^{M_{h}}_{0}$ is $m$-isomorphic to $S^{M_{l}}_{0}$;
      \item[2.] Let $\langle s_{h},s_{l}\rangle \in B_{2}$, where $s_{l}$ is a situation reached from $S^{M_{l}}_{0}$ via executing $anyllps$. then $s_{h}$ is $m$-isomorphic to $s_{l}$. Let $A(\vec{x})$ be an arbitrary HL action. By item 4, if $m(A(\vec{x}))$ is executable in the situation $s_{l}$, then we have $M_{l},v[s/s_{l}]\models \forall \vec{x}. Term(m(A(\vec{x})),s,C_{l})$. By item 3, if $s_{l}$ satisfies $\forall \vec{x}. m(\Pi_{A}(\vec{x},s_{l}))$, then there exists a situation $s'_{l}$ $R$-related to $s_{l}$ satisfies $\forall \vec{x},\exists s''. Do(m(A(\vec{x})),s',s'')$. %We let $s'_{h}$, which satisfies $Do(A(\vec{x}),s_{h},s'_{h})$, $B_{2}$-related to any situation the situation $s''$.
          Then $s_{h}$ and $s'_{l}$ are $m$-isomorphic, and we replace $\langle s_{h},s_{l} \rangle$ in $B_{2}$ by $\langle s_{h},s'_{l}\rangle $. %Then $s'_{h}$ satisfies $P(\vec{y})$ iff $s_{h}$ satisfies $\phi_{P,A_{d}}(\vec{x},\vec{y},\vec{u})$ iff $s'_{l}$ satisfies $m(\phi_{P,A_{d}}(\vec{x},\vec{y},\vec{u}))$.
          In addition, for each situation $s'_{h}$ satisfies $\vec{x}. Do(A(\vec{x}),s_{h},s'_{h})$ and all the situations $s''_{l}$ that satisfy $\vec{x}. Do(m(A(\vec{x})),s_{l},s''_{l})$, we add all the pairs $\langle s'_{h},s''_{l}\rangle$ in $B_{2}$.
          By item 5 (the case of item 6 can be discussed similarly), for any execution $A_{d}(\vec{x},\vec{u})$ of $A(\vec{x})$, $s''_{l}$ satisfies $\forall \vec{y}. m(P(\vec{y},s''_{l}))$ iff $s_{l}$ satisfies $\forall \vec{x},\vec{y},\exists \vec{u}. m(\Pi_{A_{d}}(\vec{x},\vec{y},\vec{u},s_{l}))$ iff $s_{h}$ satisfies $\forall \vec{x},\vec{y},\exists \vec{u}.\Pi_{A_{d}}(\vec{x},\vec{y},\vec{u},s_{h})$. For the situations $s'_{h}$ that do not satisfy $\forall \vec{y}. P(\vec{y},s'_{h})$, we then delete all the pairs $\langle s'_{h},s''_{l}\rangle$ from $B_{2}$.

          %Then we have $s''_{l}$ satisfies $m(P(\vec{y},s''_{l}))$ iff $s_{l}$ satisfies $m(\Pi_{A_{d}}(\vec{x},\vec{y},\vec{u},s_{l}))$ iff $s'_{h}$ satisfies $P(\vec{y},s'_{h})$. Thus, we have that $s'_{h}$ is $m$-isomorphic to $s''_{l}$, and we add $\langle s'_{h}, s''_{l} \rangle $ in $B_{2}$.
      \item[3.] $B_{2}$ is a $m$-back-simulation relation follows its construction coupled with item 7.
   \end{itemize}

Furtherfore, if $\langle s_{h},s_{l}\rangle \in B_{2}$ and $s_{h}$ satisfies $G_{h}$, then $s_{l}$ satisfies $G_{l}$. Since $s_{h}$ and $s_{l}$ are $m$-isomorphic, $s_{l}$ satisfies $m(G_{h})$, by Item 8, $s_{l}$ satisfies $G_{l}$.

In conclusion, according to the definition of sound abstraction on theory level, we can get that $\mathcal{G}_{h}$ is a sound abstraction of $\mathcal{G}_{l}$.
\end{proof}

Unfortunately, we are not able to give a proof-theoretic characterization for complete abstractions. Nonetheless, we give the following characterization where
Item 2 says that $M_l^1$ satisfies for any executable refinement of a HL situation, $\psi_T$, $\xi_P$, and $\xi_f$ hold, and the LL goal implies the mapped HL goal, and
the executability of the refinement of any HL action implies its mapped precondition; Item 3 says that $M_l^2$ satisfies that there exists a set $P$ of situations including the initial situation such that for any $P$-situation,  $\psi_T$, $\xi_P$, and $\xi_f$ hold, and the LL goal implies the mapped HL goal, and
the mapped precondition of any HL action implies that its refinement is executable and leads to a $P$-situation.

\begin{theorem}
$\mathcal{G}_{h}$ is a complete $m$-abstraction of $\mathcal{G}_{l}$ iff
   for any model $M_{h}$ of $\mathcal{D}_{h}$, there exists a model $M^{1}_{l}$ of $\mathcal{D}_{l}$ such that:

  \begin{enumerate}
    \item $S^{M_{h}}_{0}\sim_{m} S^{M^{1}_{l}}_{0}$;
    \item $M^{1}_{l}\models \forall s. Do(anyllps,S_{0},s)\supset \\
    \hspace*{1.35em}\bigwedge_{A\in \mathcal{A}_{h}} \forall \vec{x},s'. Do(m(A(\vec{x})),s,s')\supset m(\Pi_A(\vec{x},s))$;
    \item $M^{1}_{l}\models \forall s. Do(anyllps,S_{0},s)\supset \\
    \hspace*{1.35em}\bigwedge_{A\in \mathcal{A}_{h}}\forall \vec{x},s'. Do(m(A(\vec{x})),s,s') \supset \\
    \hspace*{2.7em}\bigwedge_{A\in \mathcal{A}_{h}}\forall \vec{x}. \textit{Term}(m(A(\vec{x})),s,C_l)$;
    \item $M^{1}_{l}\models \forall s. Do(anyllps,S_{0},s)\supset \xi_P$; %\bigwedge_{A\in \mathcal{A}_{h}}\forall \vec{x},s'. Do(m(A(\vec{x})),s,s')$\\ \hspace*{3em}$\supset \exists \vec{u}.\bigwedge_{P\in \mathcal{P}_{h}}[\forall \vec{y}. m(P(\vec{y},s'))\equiv m(\phi_{P,A_{d}}(\vec{y},\vec{x},\vec{u}, s))]$;
    \item $M^{1}_{l}\models \forall s. Do(anyllps,S_{0},s)\supset \xi_f$;%\bigwedge_{A\in \mathcal{A}_{h}}\forall \vec{x},s'. Do(m(A(\vec{x})),s,s')$\\         \hspace*{3em}$\supset \exists \vec{u}.\bigwedge_{f\in \mathcal{F}_{h}}[\forall \vec{y}, z. m(f(\vec{y},s')=z)\equiv m(\psi_{f,A_{d}}(\vec{y},z,\vec{x},\vec{u},s))]$,
    \item $M^{1}_{l}\models \forall s. Do(anyllps,S_{0},s)\supset G_{l}[s]\supset m(G_{h})[s]$;
    \item if  $M^{1}_{l}\models \exists h_{l}.\textit{Infexe}(anyllps, h_{l}, C_{l})$, then $ M_h \models  \exists h_{h}.\textit{Infexe}(anyhlas, h_{h}, C_{h})$,
\end{enumerate}
and there is another model $M^{2}_{l}$ of $\mathcal{D}_{l}$ and a situation set $P$ of $M^{2}_{l}$ such that:
\begin{enumerate}
    \item[8] $S^{M_{h}}_{0}\sim_{m} S^{M^{2}_{l}}_{0}$;
    \item[9] $M^{2}_{l}\models P(S_0) \land  \forall s.P(s)\supset  \\
    \hspace*{1.35em}\bigwedge_{A\in \mathcal{A}_{h}}\forall \vec{x}.m(\Pi_A(\vec{x},s)) \supset \\
    \hspace*{2.7em}\exists s'. Do(m(A(\vec{x})),s,s')\wedge P(s')$;
    \item[10] $M^{2}_{l}\models P(S_0) \land  \forall s.P(s)\supset \\
    \hspace*{1.35em}\bigwedge_{A\in \mathcal{A}_{h}}\forall \vec{x},s'. Do(m(A(\vec{x})),s,s')\land P(s') \supset  \psi_T$; %\bigwedge_{A\in \mathcal{A}_{h}}\forall \vec{x}. \textit{Term}(m(A(\vec{x})),s,C_l)$;
    \item[11] $M^{2}_{l}\models P(S_0) \land  \forall s.P(s)\supset \xi_P$; %\bigwedge_{A\in \mathcal{A}_{h}}\forall \vec{x},s'. Do(m(A(\vec{x})),s,s') \land P(s')$\\ \hspace*{3em}$\supset \exists \vec{u}.\bigwedge_{P\in \mathcal{P}_{h}}[\forall \vec{y}. m(P(\vec{y},s'))\equiv m(\phi_{P,A_{d}}(\vec{y},\vec{x},\vec{u}, s))] $;
    \item[12] $M^{2}_{l}\models P(S_0) \land  \forall s.P(s)\supset \xi_f$; %\bigwedge_{A\in \mathcal{A}_{h}}\forall \vec{x},s'. Do(m(A(\vec{x})),s,s')  \land P(s') $\\        \hspace*{3em}$\supset \exists \vec{u}.\bigwedge_{f\in \mathcal{F}_{h}}[\forall \vec{y}, z. m(f(\vec{y},s')=z)\equiv m(\psi_{f,A_{d}}(\vec{y},z,\vec{x},\vec{u},s))]$,
    \item[13] $M^{2}_{l}\models P(S_0) \land  \forall s.P(s)\supset G_{l}[s] \supset m(G_{h})[s]$;
    \item[14] if  $ M_h \models  \exists h_{h}.\textit{Infexe}(anyhlas, h_{h}, C_{h})$, then $M^{2}_{l}\models \exists h_{l}.\textit{Infexe}(anyllps, h_{l}, C_{l})$.
  \end{enumerate}
\end{theorem}

\begin{proof}
   Firstly, we prove the \textbf{only-if direction}:

\begin{enumerate}
  \item[1.] According to the definition of complete abstraction on theory level, we know that for each high-level model $M_{h}$ of $\mathcal{D}_{h}$, there exist a LL model $M^{1}_{l}$, such that $M_{h}$ is a complete abstraction of $M^{1}_{l}$ via a $m$-simulation relation $B_{1}$. Then, we have $S^{M_{h}}_{0}\sim_{m} S^{M^{1}_{l}}_{0}$;
\end{enumerate}

\begin{enumerate}
  \item[2.] Let $s_{l}$ be a situation of $M^{1}_{l}$ which satisfies the formula $Do(anyllps,  S_{0},s)$, then based on the definition of $m$-simulation relation, we have that there is a situation $s_{h}$ of $M_{h}$ such that $\langle s_{h},s_{l} \rangle \in B_{1}$. Thus, for any HL action $A(\vec{x})$, $s_{h}$ satisfies $\forall \vec{x}. \Pi_{A}(A(\vec{x}),s)$ iff $s_{l}$ satisfied $\forall \vec{x}. m(\Pi_{A}(A(\vec{x}),s))$. Furthermore, if $M^{1}_{l},v[s/s_{l}]\models \forall \vec{x}. \exists s'_{l}. Do(m(A(\vec{x})),s,s'_{l})$, then $M_{h},v[s/s_{h}]\models \forall \vec{x}. \exists s'_{h}. Do(A(\vec{x}),s,s'_{h})$, which means that $M_{h},v[s/s_{h}]\models \forall \vec{x}. \Pi_{A}(A(\vec{x}),s)$. Therefore, $M^{1}_{l},v[s/s_{l}]\models \forall \vec{x}. m(\Pi_A(\vec{x},s))$.
\end{enumerate}

\begin{enumerate}
  \item[3.] Let $s_{l}$ be a situation of $M^{1}_{l}$ which satisfies the formula $Do(anyllps,  S_{0},s)$, then based on the definition of $m$-simulation relation, we have that there is a situation $s_{h}$ of $M_{h}$ such that $\langle s_{h},s_{l} \rangle \in B_{1}$, and $M^{1}_{l},v[s/s_{l}]\models \forall \vec{x}. Term(A(\vec{x}),s,C_{l})$ for any HL action $A(\vec{x})$.
\end{enumerate}

\begin{enumerate}
  \item[4.] Let $s_{l}$ be a situation of $M^{1}_{l}$ which satisfies the formula $Do(anyllps,  S_{0},s)$, then based on the definition of $m$-simulation relation, we have that there is a situation $s_{h}$ of $M_{h}$ such that $\langle s_{h},s_{l} \rangle \in B_{1}$. For each HL action $A(\vec{x})$, and any situation $s'_{l}$, if $M^{1}_{l},v[s/s_{l}, s'/s'_{l}]\models \forall \vec{x}. Do(m(A(\vec{x})),s,s')$, then there exists an execution $A(\vec{x}, \vec{u})$ of $A(\vec{x})$, and a situation $s'_{h}$, such that $M_{h},v[s/s_{h},s'/s'_{h}]\models \forall \vec{x}, \exists \vec u. Do(A_{d}(\vec{x}, \vec{u}),s,s')$, and $\langle s'_{h},s'_{l} \rangle \in B$. Thus, $s_{l}$ satisfies $\forall \vec{x},\vec{y},\exists \vec{u}. m(\phi_{P,A_{d}}(\vec{x},\vec{y},\vec{u},s))$ iff $s_{h}$ satisfies $\forall \vec{x},\vec{y},\exists \vec{u}.\phi_{P,A_{d}}(\vec{x},\vec{y},\vec{u},s)$ iff $s'_{h}$ satisfies $\forall \vec{y}. P(\vec{y},s')$ iff $s'_{l}$ satisfies $\forall \vec{y}. m(P(\vec{y},s'))$.
\end{enumerate}

\begin{enumerate}
  \item[5.] The proof of item 5 is similar to item 4.
\end{enumerate}

\begin{enumerate}
  \item[6.] Let $s_{l}$ be a situation of $M^{1}_{l}$ which satisfies the formula $Do(anyllps,  S_{0},s)$, then based on the definition of $m$-simulation relation, we have that there is a situation $s_{h}$ of $M_{h}$ such that $\langle s_{h},s_{l} \rangle \in B_{1}$. Suppose $s_{l}$ satisfies $G_{l}$, then $s_{h}$ satisfies  $G_{h}$ based on the weak complete abstraction on theory level. Thus, we have that $s_{l}$ satisfies $m(G_{h})$.
\end{enumerate}

\begin{enumerate}
  \item[7.] According to the definitions of complete abstraction and $m$-simulation, if there exists a LL infinite history $h_{l}$ such that $M^{1}_{l},v[h/h_{l}]\models Infexe(anyllps, h_{l}, C_{l})$, we can construct an HL infinite history $h_{h}$ such that $\langle h_{h},h_{l} \rangle \in B_{1}$. Then we can get $M_{h},v[h/h_{h}]\models Infexe(anyllps, h_{h}, C_{h})$.
\end{enumerate}

    For item 8-14, we firstly construct a situation set $P$ of $M^{2}_{l}$ as follows: (i) $S^{M^{2}_{l}}_{0}\in P$; (ii) for any LL program $anyllps$, if there is a situation $s_{l}$ which satisfies $Do(anyllps, S^{M^{2}_{l}}_{0},s_{l})$, then we let $s_{l}\in P$. Now we prove item 8-14.
\vspace{2mm}

\begin{enumerate}
  \item[8.] According to the definition of complete abstraction on theory level, we know that for each high-level model $M_{h}$ of $\mathcal{D}_{h}$, there exist a LL model $M^{2}_{l}$, such that $M_{h}$ is a sound abstraction of $M^{2}_{l}$ of via a $m$-back-simulation relation $B_{2}$. Then, we have  $S^{M_{h}}_{0}\sim_{m} S^{M^{2}_{l}}_{0}$;
\end{enumerate}

\begin{enumerate}
  \item[9.] Let $s_{h}$ be a reachable situation of $M_{h}$ via executing $anyhlas$, then according to the construction of $P$, we can find a situation $s_{l}\in P$ of $M^{2}_{l}$, such that $\langle s_{h},s_{l} \rangle \in B_{2}$. Furthermore, for any HL action $A(\vec{x})$, we have that if $s_{h}$ satisfies $\forall \vec{x}. \Pi_{A}(A(\vec{x}),s)$, then there exists a situation $s'_{l}$ such that $s_{l}$ satisfies $\forall \vec{x}. Do(m(A(\vec{x})),s_{l},s'_{l})$ and $s'_{l}\in P$.
\end{enumerate}

\begin{enumerate}
  \item[10.] Let $s_{h}$ be a reachable situation of $M_{h}$ via executing $anyhlas$, then according to the construction of $P$, we have that there exists a situation $s_{l}\in P$ of $M^{2}_{l}$, such that $\langle s_{h},s_{l} \rangle \in B_{2}$. Then for any HL action $A(\vec{x})$, based on the definition of $m$-back-simulation, we can get that $M^{2}_{l},v[s/s_{l}]\models \forall \vec{x}. Term(A(\vec{x}),s,C_{l})$.

\end{enumerate}

\begin{enumerate}
  \item[11.] Let $s_{h}$ be a reachable situation of $M_{h}$ via executing $anyhlas$, then according to the construction of $P$, we have that there exists a situation $s_{l}\in P$ of $M^{2}_{l}$, such that $\langle s_{h},s_{l} \rangle \in B_{2}$. For each HL action $A(\vec{x})$, if there exists an execution $A(\vec{x},\vec{u})$ of $A(\vec{x})$, and a situation $s'_{h}$, such that $M_{h},v[s/s_{h},s'/s'_{h}]\models \forall \vec{x}, \exists \vec u. Do(A_{d}(\vec{x},\vec{u}),s,s')$, then based on the definition of $m$-back-simulation and the construction of $P$, we can get that there exists a situation $s'_{l}$ such that $M^{2}_{l},v[s/s_{l}, s'/s'_{l}]\models \forall \vec{x}. Do(m(A(\vec{x})),s,s')\land P(s')$, and $\langle s'_{h}, s'_{l} \rangle \in B_{2}$. Then, $s_{l}$ satisfies $\forall \vec{x},\vec{y},\exists \vec{u}.  m(\phi_{P,A_{d}}(\vec{x},\vec{y},\vec{u},s))$ iff $s_{h}$ satisfies $\forall \vec{x},\vec{y},\exists \vec{u}.  \phi_{P,A_{d}}(\vec{x},\vec{y},\vec{u},s)$ iff $s'_{h}$ satisfies $\forall \vec{y}. P(\vec{y},s')$ iff $s'_{l}$ satisfies $\forall \vec{y}. m(P(\vec{y},s'))$.
\end{enumerate}

\begin{enumerate}
  \item[12.]    The proof of item 12 is similar to item 11.
\end{enumerate}

\begin{enumerate}
  \item[13.] Let $s_{h}$ be a reachable situation of $M_{h}$ via executing $anyllps$, then based on the construction of $P$, there is a situation $s_{l}$ of $M^{2}_{l}$, such that $\langle s_{h},s_{l} \rangle \in B_{2}$. Suppose $s_{l}$ satisfies $G_{l}$, then according to the condition 2 of the definition of complete abstraction on theory level, we have that $s_{h}$ satisfies $G_{h}$. Thus, $s_{l}$ satisfies $m(G_{h})$,
\end{enumerate}

\begin{enumerate}
  \item[14.] According to the definitions of sound abstraction and $m$-back-simulation, if there exists a HL infinite history $h_{h}$ such that $M_{h},v[h/h_{h}]\models Infexe(anyhlas, h_{h}, C_{h})$, then we can construct an LL infinite history $h_{l}$ such that $\langle h_{h},h_{l} \rangle \in B_{2}$. Then we can get that $M^{2}_{l},v[h/h_{l}]\models Infexe(anyllps, h_{l}, C_{l})$.
\end{enumerate}

Secondly, we prove the \textbf{if direction}:
\vspace{2mm}

For $M^{1}_{l}$ and $M_{h}$, we firstly prove that $M_{h}$ is a complete abstraction of $M^{1}_{l}$. For this purpose, we construct a $m$-simulation relation $B_{1}$ between the situation and infinite history domains of $M_{h}$ and $M^{1}_{l}$ as follows:
    \begin{itemize}
     \item[1.]   $\langle S^{M_{h}}_{0}, S^{M^{1}_{l}}_{0}\rangle \in  B_{1}$,  i.e., $S^{M_{h}}_{0}$ is $m$-isomorphic to $S^{M^{1}_{l}}_{0}$;
     \item[2.] Let $\langle s_{h},s_{l}\rangle \in B_{1}$, where $s_{l}$ is a situation reached from  $S^{M^{1}_{l}}_{0}$ via executing $anyllps$, then $s_{h}$ is $m$-isomorphic to $s_{l}$. Let $A(\vec{x})$ be an arbitrary HL action. %By item 2, we have that if $s_{l}$ satisfies $poss(m(A(\vec{x})),s_{l})$, then $s_{h}$ satisfies $\Pi_A(\vec{x},s_{h})$.
         By item 3, if $m(A(\vec{x}))$ is executable in the situation $s_{l}$, then we have $M_{l},v[s/s_{l}]\models \forall \vec{x}. Term(m(A(\vec{x})),s,C_{l})$. By item 2, if $m(A(\vec{x}))$ is executable in $s_{l}$, then $A(\vec{x})$ is executable in $s_{h}$. For each situation $s'_{l}$ satisfies $\forall \vec{x}.Do(m(A(\vec{x})),s_{l},s'_{l})$ and all the situations $s'_{h}$ satisfies $\forall \vec{x}. Do(A(\vec{x}),s_{h},s'_{h})$, we add all the pairs $\langle s'_{h},s'_{l} \rangle$ in $B_{1}$.  By item 4 (the case of item 5 can be discussed similarly), for any execution $A_{d}(\vec{x},\vec{u})$ of $A(\vec{x})$, $s'_{l}$ satisfies $\forall \vec{y}. m(P(\vec{y},s'_{l}))$ iff $s_{l}$ satisfies $\forall \vec{x},\vec{y},\exists \vec{u}. m(\Pi_{A_{d}}(\vec{x},\vec{y},\vec{u},s_{l}))$ iff $s_{h}$ satisfies $\forall \vec{x},\vec{y},\exists \vec{u}. \Pi_{A_{d}}(\vec{x},\vec{y},\vec{u},s_{h})$. For the situations $s'_{h}$ that do not satisfy $\forall \vec{y}. P(\vec{y},s'_{h})$, we then delete all the pairs $\langle s'_{h},s'_{l}\rangle$ from $B_{1}$.

         %Then we have that $s'_{l}$ satisfies $m(P(\vec{y},s'_{l}))$ iff $s_{l}$ satisfies $m(\Pi_{A_{d}}(\vec{x},\vec{y},\vec{u},s_{l}))$ iff $s_{h}$ satisfies $\Pi_{A_{d}}(\vec{x},\vec{y},\vec{u},s_{h})$ iff $s'_{h}$ satisfies $P(\vec{y},s'_{h})$. Thus, we have that $s'_{h}$ is $m$-isomorphic to $s'_{l}$, and we add $\langle s'_{h}, s'_{l} \rangle $ in $B_{1}$.
     \item[3.]  $B$ is a $m$-simulation relation follows its construction coupled with item 7.
   \end{itemize}
Furthermore, if $\langle s_{h},s_{l}\rangle \in B_{1}$ and $s_{l}$ satisfies $G_{l}$, then $s_{h}$ satisfies $G_{h}$. Since $s_{h}$ and $s_{l}$ are $m$-isomorphic, $s_{l}$ satisfies $G_{l}$, by Item 6, $s_{l}$ satisfies $m(G_{h})$, thus $s_{h}$ satisfies $G_{h}$.
\vspace{2mm}

For $M^{2}_{l}$ and $M_{h}$, we prove that $M_{h}$ is a sound abstraction of $M^{2}_{l}$. For this purpose, in the following, we construct a $m$-back-simulation relation $B_{2}$ between the situation and infinite history domains of $M_{h}$ and $M^{2}_{l}$, particularly, we consider the situation set $P$ of $M^{2}_{l}$ we constructed above.
    \begin{itemize}
     \item[1.]   $\langle S^{M_{h}}_{0}, S^{M^{2}_{l}}_{0}\rangle \in  B_{2}$,  i.e., $S^{M_{h}}_{0}$ is $m$-isomorphic to $S^{M^{2}_{l}}_{0}$, where $S^{M^{2}_{l}}_{0}\in P$;
     \item[2.] Let $\langle s_{h},s_{l}\rangle \in B_{2}$, where $s_{l}\in P$ is a situation reached from  $S^{M^{2}_{l}}_{0}$ via executing $anyllps$, then $s_{h}$ is $m$-isomorphic to $s_{l}$. Let $A(\vec{x})$ be an arbitrary HL action. %By item 2, we have that if $s_{l}$ satisfies $poss(m(A(\vec{x})),s_{l})$, then $s_{h}$ satisfies $\Pi_A(\vec{x},s_{h})$.
         By item 10, if $m(A(\vec{x}))$ is executable in the situation $s_{l}$, then we have $M_{l},v[s/s_{l}]\models Term(m(A(\vec{x})),s,C_{l})$. By item 9, if $A(\vec{x})$ is executable in $s_{h}$, then $m(A(\vec{x}))$ is executable in $s_{l}$. For each situation $s'_{h}$ satisfies $\forall \vec{x}. Do(A(\vec{x}),s_{h},s'_{h})$ and all the situations $s'_{l}\in P$ that satisfy $\forall \vec{x}. Do(m(A(\vec{x})),s_{l},s'_{l})$, we add all the pairs $\langle s'_{h},s'_{l}\rangle$ in $B_{2}$.  By item 11 (the case of item 12 can be discussed similarly), for any execution $A_{d}(\vec{x},\vec{u})$ of $A(\vec{x})$, $s'_{l}$ satisfies $\forall \vec{y}. m(P(\vec{y},s'_{l}))$ iff $s_{l}$ satisfies $\forall \vec{x},\vec{y},\exists \vec{u}. m(\Pi_{A_{d}}(\vec{x},\vec{y},\vec{u},s_{l}))$ iff $s_{h}$ satisfies $\forall \vec{x},\vec{y},\exists \vec{u}.\Pi_{A_{d}}(\vec{x},\vec{y},\vec{u},s_{h})$. For the situations $s'_{h}$ that do not satisfy $\forall \vec{y}. P(\vec{y},s'_{h})$, we then delete all the pairs $\langle s'_{h},s'_{l}\rangle$ from $B_{2}$.

         %Then we have that $s'_{l}$ satisfies $m(P(\vec{y},s'_{l}))$ iff $s_{l}$ satisfies $m(\Pi_{A_{d}}(\vec{x},\vec{y},\vec{u},s_{l}))$ iff $s'_{h}$ satisfies $P(\vec{y},s'_{h})$. Thus, we have that $s'_{h}$ is $m$-isomorphic to $s'_{l}$, and we add $\langle s'_{h}, s'_{l} \rangle $ in $B_{1}$.
     \item[3.]  $B_{2}$ is a $m$-back-simulation relation follows its construction coupled with item 14.
   \end{itemize}
Furthermore, if $\langle s_{h},s_{l}\rangle \in B_{2}$ and $s_{l}$ satisfies $G_{l}$, then $s_{h}$ satisfies $G_{h}$. Since $s_{h}$ and $s_{l}$ are $m$-isomorphic, $s_{l}$ satisfies $G_{l}$, by Item 13, $s_{l}$ satisfies $m(G_{h})$, thus $s_{h}$ satisfies $G_{h}$.

\end{proof}

\section{Implementation Theory}

In this section, by introducing some restrictions in Theorem \ref{s-abs-nd-case}, we first give a sufficient condition for sound abstractions that is first-order verifiable. Then we discuss how to verify the sufficient condition with first-order theorem prover when the g-planning problem abstractions are QNPs.

We first discuss how to obtained a sufficient condition for sound abstractions from Theorem 1.
%Although Theorem \ref{s-abs-nd-case} provides us with a necessary and sufficient condition for sound abstractions, but it is not first-order verifiable. Thus, standard verification techniques can not provide convenience for us. For this reason,
In the following,
we use  $\models_{fo}$ to denote classic first-order entailment:

(1). Situations that satisfy $Do(anyllps, S_{0}, s)$ are all executable. State constraints are formulas that hold in all executable situations. Given a BAT $\mathcal{D}$ and a formula $\phi(s)$, $\phi(s)$ is a state constraint for $\mathcal{D}$ if $\mathcal{D}\models \forall s.Exec(s)\supset \phi(s)$. Thus, for getting a sufficient condition for sound abstractions, we replace the condition $Do(anyllps, S_{0}, s)$  by providing LL state constraints in tasks 2, 3, 4 and 6 in Theorem 1. We use $\mathcal{D}_{sc}$ to denote a set of state constraints, and abuse $\mathcal{D}_{sc}$ as the conjunction of its elements.

(2). Determining whether a given program terminates is the quintessential undecidable problem. We assume that any HL action refinement does not involve iteration. Then, the formula $\psi_{T}$ in task 4 is true trivially.

(3). We restrict that each HL action is deterministic, thus, we can ignore trajectory constraints, and then ignore task 5.

(4). Since task 2 involves the reasoning about LL  programs, we use existentially extended regression to compute executability conditions of Golog programs. Given a program $\delta$ and a situation $s$, the executability condition $pre(\delta,s)$ of $\delta$ in the situation $s$ can be computed as $\mathcal{R}^{E}[\top(s), A(\vec{x})]$. For space saving, we use the the following abbreviation:

\vspace{1mm}
$\xi_{A} \doteq \bigwedge_{A\in \mathcal{A}_{h}}\forall \vec{x}.pre(m(A(\vec{x})),s)\equiv m(\Pi_{A}(\vec{x},s)).$

\vspace{1mm}

\noindent Again, for getting a sufficient condition for sound abstractions, we verify the following stronger task which can implies both task 2 and task 3: $\models_{fo} \forall s. \mathcal{D}_{sc}(s) \supset \xi_{A}$,

\vspace{1mm}

 (5). For further discussion, we introduce two more abbreviations as follows:

\vspace{1mm}
\noindent
$ \zeta_P \doteq \bigwedge_{A\in \mathcal{A}_{h}} \forall \vec{x}. pre(m(A(\vec{x}),s)) $\\
\hspace*{2.2em}	$\supset \bigwedge_{P\in \mathcal{P}_{h}} \forall \vec{y}.[\mathcal{R}^{E}[ m(P(\vec{y}))[s], m(A(\vec{x}))]  $\\
\hspace*{2.2em} $\supset m(\phi_{P,A}(\vec{x}, \vec{y}, s))] \land [m(\phi_{P,A}(\vec{x}, \vec{y}, s)) $  \\
\hspace*{2.2em} $\supset\mathcal{R}^{U}[ m(P(\vec{y}))[s],m(A(\vec{x}))]$.

\vspace{1mm}
\noindent
 $\zeta_f \doteq \bigwedge_{A\in \mathcal{A}_{h}}\forall \vec{x}. pre(m(A(\vec{x}),s))$\\
\hspace*{2.2em}	$\supset \bigwedge_{f\in \mathcal{F}_{h}}\forall \vec{y}, z.[\mathcal{R}^{E}[ m(f(\vec{y})=z)[s], m(A(\vec{x}))]  $\\
\hspace*{2.2em}	$\supset m(\psi_{f,A}(\vec{x},\vec{y},z,s))] \land [m(\psi_{f,A}(\vec{x},\vec{y},z,s))$  \\
\hspace*{2.2em}	$\supset \mathcal{R}^{U}[m(f(\vec{y})=z)[s],m(A(\vec{x}))]].$

 \vspace{1mm}
\noindent Based on (1) and the two extended regression definitions, we can get that task 4 is equivalent to:
\begin{equation*}
	\setlength{\abovedisplayskip}{1mm}
	\setlength{\belowdisplayskip}{1mm}
    \models_{fo} \forall s. \mathcal{D}_{sc}(s) \supset  \zeta_P \land \zeta_f.
\end{equation*}
% \begingroup
% \addtolength{\jot}{-1mm}
% \begin{flalign*}
%   &\mathcal{D}_{l}\models\forall \vec{x},\vec{y}, z, s. \mathcal{D}_{sc}(s)\land pre(m(A(\vec{x}),s))\supset \xi_2(s)&
% %   &\ \ [\mathcal{R}^{E}[ m(f(\vec{y},s')=z), m(A(\vec{x}))] \supset m(\psi_{f,A}(\vec{x},\vec{y},z,s))]\land &\\
% %   &\ \ [m(\psi_{f,A}(\vec{x},\vec{y},z,s)) \supset \mathcal{R}^{U}[m(f(\vec{y},s)=z),m(A(\vec{x}))]].&
% \end{flalign*}
% \endgroup

% For the item 3 in Theorem \ref{s-abs-nd-case}, for each HL functional fluent $f\in\mathcal{F}_{h}$, and HL action $A\in \mathcal{A}_{h}$, similar to task 3, we prove the following task:%This item means that $\mathcal{D}_{l}$ must entail the mapped successor state axioms provided that the LL program can be executable successfully. % we split item 4 into the.
% \begingroup
% \addtolength{\jot}{-1mm}
% \begin{flalign*}
%   &\mathcal{D}_{l}\models\forall \vec{x},\vec{y}, z, s, s'. \mathcal{D}_{sc}(s)\land pre(m(A(\vec{x}),s))\supset &\\
%   &\ \ [\mathcal{R}^{E}[ m(f(\vec{y},s')=z), m(A(\vec{x}))] \supset m(\psi_{f,A}(\vec{x},\vec{y},z,s))]\land &\\
%   &\ \ [m(\psi_{f,A}(\vec{x},\vec{y},z,s)) \supset \mathcal{R}^{U}[m(f(\vec{y},s')=z),m(A(\vec{x}))]].&
% \end{flalign*}
% \endgroup

We have the following result based on the analysis above:

\begin{corollary}\label{implementation_theory}
  Given a g-planning problem $\mathcal{G}_{l}$ and its abstraction $\mathcal{G}_{h}$, suppose that all the HL actions are deterministic and their refinements do not involve iteration, then $\mathcal{G}_{h}$ is a sound abstraction of $\mathcal{G}_{l}$ if:
  \begin{enumerate}[leftmargin=*]
    \item $\mathcal{D}_{l}^{S_{0}}\models_{fo}  m(\phi)$, where $\phi\in \mathcal{D}_{h}^{S_{0}}$;
    \item $\models_{fo} \forall s. \mathcal{D}_{sc}(s)\supset \xi_A$;
    \item $\models_{fo} \forall s. \mathcal{D}_{sc}(s) \supset  \zeta_P \land \zeta_f$;
    \item $\models_{fo} \forall s. \mathcal{D}_{sc}(s)\land m(G_{h})[s]\supset G_{l}[s]$.
  \end{enumerate}
\end{corollary}

The soundness of abstractions can be verified by feeding all 4 tasks above into a theorem prover. However, existing first-order theorem provers do not support counting or transitive closure. QNPs are popular abstraction models for g-planning problems. We now turn to discuss how to verify QNP abstractions. In particular, we develop methods to handle the counting and transitive closure.

QNPs are extensions of classical planning problems. Compared to classical planning problems, state variables of QNPs contain non-negative numerical variables $n$. These variables introduce the non-propositional atoms $n = 0$ and their negations $n>0$. These literals can appear in the initial situations, action preconditions, and goals of QNPs. The effects of actions of QNPs on a numerical variable $n$ can be increments or decrements denoted by $n\uparrow$ and $n\downarrow$, respectively.

%\begin{definition}\rm
%  A \text{QNP} is a tuple $Q=\langle F, V, I, O, G \rangle $, where $F$ and $V$ are sets of propositional and numerical variables respectively. $I$, and $G$ denote the initial, and goal conditions, respectively. $O$ is a set of actions $a$ with preconditions, propositional numerical and numerical effects that are denoted as $\text{Pre}(a)$, $\text{Eff}(a)$ and $N(a)$ respectively. The $F$-literals can appear in $I$, $G$, $\text{Pre}(a)$, and $\text{Eff}(a)$, while $V$-literals can appear in $I$, $G$, $\text{Pre}(a)$. The numerical effects $N(a)$ only contain special atoms of the form $n\uparrow$ or $n\downarrow$ for the variable $n\in V$. Actions with the $n\downarrow$ effect must feature the precondition $n>0$ for any variable $n\in V$.
%\end{definition}
\vspace{+1mm}
\noindent {\bf Example \ref{eg-clearing a block} cont'd.}
An abstraction for the g-planning problem $ClearA$ is a QNP $Q=\langle F, V, I, O, G \rangle $, where $F=\{H\}$ contains a propositional variable $H$ that means the agent holding a block, $V=\{n\}$ contains a numerical variable $n$ that means the number of blocks above $A$. The initial state $I$ is $n>0\land \neg H$, the goal state $G$ is $n=0$, and the actions $O=\{pickabove, putaside\}$ are defined as follows:

\vspace{1mm}
\noindent \hspace*{3mm}$pickabove:\langle \neg H\land n>0;H,n\downarrow \rangle;\ putaside:\langle H;\neg H \rangle.$

\vspace{1mm}
 In QNP abstraction case, the verification for task 3 in Theorem \ref{implementation_theory} can be more specific . For each HL relational fluent $p\in\mathcal{P}_{h}$, and HL action $A\in \mathcal{A}_{h}$, $A$ makes $p$ true (the false case can be discussed similarly) means that in LL, all the executions of $m(A)$ make $m(p)$ true, then we have the  task $3^{*}$:

\vspace{1mm}
\noindent \hspace*{0.5em}$\models_{fo} \forall s. \mathcal{D}_{sc}(s) \land $\\
   \hspace*{1.5em}$ \forall \vec{x}. pre(m(A(\vec{x})),s)\supset \mathcal{R}^{U}[m(p(s)),m(A(\vec{x}))].$
\vspace{1mm}

For the verification tasks involving functional fluents, we only discuss a restricted case, that is,  effects of actions on functional fluents are $\pm 1$. Given a functional fluent $f\in\mathcal{F}_{h}$, and action $A\in \mathcal{A}_{h}$,  if $A$ makes $n$ increase by 1 (the decrease case can be discussed similarly), we prove the task $4^{*}$:

\vspace{1mm}
\noindent \hspace*{0.5em} $\models_{fo} \forall s. \mathcal{D}_{sc}(s) \land $\\
\hspace*{1.5em}$\forall \vec{x}. pre(m(A(\vec{x}),s)\supset\forall s'. Do(m(A(\vec{x})),s,s')\supset $\\
  \hspace*{1.5em}$\forall \vec{y},z. m(f(\vec{y},s')=k+1)\equiv m(\psi_{f,A}(\vec{x},\vec{y},k,s)).$

\vspace{1mm}
\noindent which means that if $m(A)$ is executable in any LL situation $s$, then for any situations $s'$ that  can arrived by executing $m(A)$ from $s$, we have $m(n)[s']=m(n)[s]+1$. Assuming that $m(n)=\#x.\phi(x)$, we have the formula $\Psi$ below that holds:

\vspace{1mm}
  \noindent $\hspace*{1em}[\exists x. \phi(x,s')\land \neg \phi(x,s)]\land[\forall x. \phi(x,s)\supset \phi(x,s')]\land$\\
  $\hspace*{1.5em}[\forall x,y. \phi(x, s)\land\neg \phi(x,s')\land \phi(y,s')\land \phi(y,s)\supset x=y].$

\vspace{1mm}
\noindent This formula says that: there exists one and only one object that makes $\phi(x)$ true from false after a program execution. Then base on Proposition \ref{prop_Uregression}, we can get the following task 4$^{\#}$, which is equivalent to task $4^{*}$:

\vspace{1mm}
\noindent \hspace*{1em}$\models_{fo} \forall s. \mathcal{D}_{sc}(s)\land $\\
\hspace*{1.5em}$\forall \vec{x}. pre(m(A(\vec{x})),s) \supset \mathcal{R}^{U}[{\Psi(s),m(A(\vec{x}))}].$

\vspace{1mm}

Transitive closure formulas are often used to define counting terms. Thus, task $4^{\#}$ may involve regression about transitive closure. In fact, the definition of one-step regression of transitive closure formulas is the same as Definition \ref{def_reg}.
%We define :

% \begin{definition}\rm
% Given a BAT $\mathcal{D}$, let $\phi$ be a transitive closure formula. We use $\Reg_{\mathcal{D}}[\phi]$ to denote the formula obtained from $\phi$ by the following steps: replace each atom $P(\vec{t}, do(\alpha, \sigma))$ with $\phi_P(\vec{t}, \alpha, \sigma)$;  replace each precondition atom $Poss(A(\vec{t}), \sigma)$ with $\Pi_{A}(\vec{t}, \sigma)$; and further simplify the result with $\at_{una}$.
% \end{definition}

\noindent {\bf Example \ref{eg-clearing a block} cont'd.}
Given the successor state axiom of fluent $on(x,y,s)$ in Example \ref{eg-clearing a block}, and a transitive closure formula $\phi$:
\begin{equation*}
	\setlength{\abovedisplayskip}{1mm}
	\setlength{\belowdisplayskip}{0mm}
    [TC_{x,y} on(x,y,do(unstack(A,B),s))](x,C)],
\end{equation*}

\vspace{1mm}
\noindent the regression result $\mathcal{R}_{\mathcal{D}}[\phi]$ of $\phi$ related to the concrete action $unstack(A,B)$ is as follows:
\begin{equation*}
	\setlength{\abovedisplayskip}{1mm}
	\setlength{\belowdisplayskip}{0mm}
    [TC_{x,y} on(x,y,s)\land (x\neq A\lor y\neq B)](x,C).
\end{equation*}

\vspace{1mm}

To handle the regression of transitive closure formulas with existing theorem provers, for a given transitive closure formula $\phi(\vec{x})$, we first get the regression result $\mathcal{R}_{\mathcal{D}}[\phi(\vec{x})]$ for $\phi(\vec{x})$.  Then we define $\mathcal{R}_{\mathcal{D}}[\phi(\vec{x})]$ as a new relation $P(\vec{x})$ and endow it with the transitivity and the minimality.

%For the , we also has Proposition \ref{def_reg} holds. In addition, transitive closure formulas involved in abstractions often have simple structures, and that can be proved by existing first-order theorem provers.
% \begin{flalign*}
%   & \mathcal{R}_{\mathcal{D}}[\phi]\equiv[TC_{x,y}x=A\land  \\
%   &\hspace*{2.5em} y=B\lor on(x,y,s)\land x\neq A](x,C).
% \end{flalign*}

\section{Verification System and Experiment Results}
%
%In this section, we introduce our verification system and experiment results.
%
%\subsection{Verification System}

Based on the discussion in Section 4, we designed a sound abstraction verification system for g-planning. The inputs of our system include a g-planning problem coupled with state constraints, a refinement mapping, and an abstraction problem. The output of our system is $True$, $False$, or $Unknown$. g-planning problems in our system take the form of STRIPS-like problems. The only extension of g-planning problems to STRIPS problems is that their initial states can be first-order formulas with transitive closure (see Example \ref{eg-clearing a block}). The formalization of QNP abstractions in our system is the same as that in \cite{BonetG20}.

The workflow of our verification system is as follows:
\vspace*{+0.5mm}
\par\noindent \textbf{Step 1:} Given the input g-planning problem, the system automatically generates the LL BAT;

\vspace*{+0.5mm}
\par\noindent \textbf{Step 2:} Based on the input abstraction problem, refinement mapping, the generated LL BAT, and LL state constraints, the system generates the verification tasks that we mention in Section 4. Concretely, the system generates task 1 for the HL and the LL initial states; task 2 for each HL action; task 3$^{*}$ for each HL relational fluent;  task 4$^{\#}$ for each HL functional fluent; and task 5 for the HL and the LL goal states.

\par\noindent \textbf{Step 3:} The system feeds all the tasks above into the Z3-solver (version 4.8.10.0) \cite{smt} for verification. If all these tasks can be verified, then the system returns $True$, else returns $False$. If the theorem prover times out, the system returns $Unknown$.

	\vspace{1mm}
	\noindent {\bf Example \ref{eg-clearing a block} cont'd.} The refinement mapping about the HL relational $H$ and functional fluent $n$ are as follows:
	
	\vspace{1mm}
	\indent $m(H)=\exists x.holding(x)$; $m(n)=\# x. on^{+}(x,A)$.

\noindent The task 1 generated by our system is as follows:

 $\exists x. on^{+}(x,A)\land ontable(A)\land \neg holding(x) $\\
\indent \hspace*{4em}$\supset \neg \exists x. holding(x)\land \exists x. on^{+}(x,A).$

Our verification system was tested on 7 domains: $ClearA$ is our running example; $OnAB$ is about achieving the goal $On(A,B)$ in Blocksworld instances where the gripper is initially empty, and the blocks $A$ and $B$ are in different towers with blocks above them; $Logistics$ involves a vehicle whose goal is to load goods from the original location and transport them to the target location; $Gripper$ involves a robot with two grippers whose goal is to move all balls from room A to B; $GetLast$ and $FindA$ are both linked list domain problems. The predicate and the action sets of these two problems are the same. The goal of $GetLast$ is to traverse all the elements in a linked list, while $FindA$ aims at finding the element $A$ in a linked list; $Corner$ contains instances that an agent needs to navigate in a rectangle grid and arrive at the point $(0,0)$ from any other points $(x,y)$. Abstractions of the problems $ClearA$, $Gripper$, and $OnAB$ come from \cite{BonetFG19}, and abstractions of the problems $Logistic$, $GetLast$, $FindA$, $Corner$ are all provided by hand.

Our experiments were run on a Windows machine with a 3.7GHz CPU and 16GB RAM, the default time limit of each subtask was 10s. We summarize the experimental results in Table 1. \textit{\#A} is the number of HL actions. \textit{\#F} is the number of HL functional fluent $F$. \textit{\#P} is the number of HL relational fluent $P$. $T$ shows the total time costs of all verification tasks. The results show that all abstractions are sound.

\begin{table}
	\renewcommand{\arraystretch}{1.3}
	\caption{Experimental Results}
	\label{table_example}
	\centering
	\begin{tabular}{ccccccc}
		\hline
		\bfseries \textit{Domain} &  \bfseries \textit{\#A}
		&  \bfseries \textit{\#F} &  \bfseries \textit{\#P}   & \bfseries \textit{T(s)} &\bfseries  \textit{Result} \\
		\hline
		ClearA     & 2  & 1  & 2  & 4.1687  & True \\
		Gripper    & 4  & 5  & 2  & 7.0014  & True \\
		Logistics  & 4  & 3  & 2  & 6.0052  & True \\
		OnAB       & 4  & 3  & 8  & 10.2191  & True \\
		GetLast    & 2  & 1  & 1  & 3.5302  & True \\
		FindA      & 2  & 1  & 1  & 3.5369  & True \\
		Corner     & 2  & 2  & 0  & 3.6378  & True \\
		\hline
	\end{tabular}
\end{table}

\section{Conclusion}
%In this paper, we proposed a method for verifying sound abstractions for g-planning problems. based on Cui et al.'s work, we firstly present the proof-theoretic characterization for sound abstraction. Then, based on the characterization, we give a sufficient condition for sound abstractions which is first-order verifiable.  To implement it, we exploit regression extensions, and develop methods to handle counting and transitive closure. Finally, we implement a sound abstraction verification system and report experimental results on several domains.

% -------------------
In g-planning, solutions of sound abstractions are those with correctness guarantees for the original problems. In this paper,  based on Cui et al.’s work,  we explored automatic verification of sound abstractions for g-planning. We gave a proof-theoretic characterization of sound abstractions for g-planning in the situation calculus. Then, we got a sufficient  condition  for  sound  abstractions  which  is  first-order verifiable.  To implement it, we exploited  regression extensions and presented methods to handle counting and transitive closure. In the future, we are interested in automatic verification concerning trajectory constraints for non-deterministic abstractions, such as FOND. We are also interested in automatic learning abstractions and abstraction revision based on the verification of sound abstraction.

\section*{Acknowledgments}
 We acknowledge support from the Natural Science Foundation of China under Grant No. 62076261.

\bibliographystyle{named}
\bibliography{ijcai22}

\end{document}